\documentclass[sigconf]{aamas}  

\AtBeginDocument{%
  \providecommand\BibTeX{{%
    \normalfont B\kern-0.5em{\scshape i\kern-0.25em b}\kern-0.8em\TeX}}}
    
\allowdisplaybreaks
\usepackage{balance}
\usepackage{amsthm}
\usepackage{thm-restate}

\usepackage{amsmath,amssymb}
\usepackage{amsfonts}
\usepackage{color}
\usepackage{enumerate}
\usepackage{flushend}
\usepackage{mathrsfs}
\usepackage{subfigure}
\usepackage{booktabs}
\usepackage{makecell}
\usepackage{setspace}
\usepackage[linesnumbered,ruled,vlined]{algorithm2e}
\SetKwComment{command}{right mark}{left mark}
\usepackage{algpseudocode}
\usepackage{amsmath}
\usepackage{cleveref}
\crefname{equation}{Eq.}{Eqs.}
\crefname{section}{Section}{Sections}
\crefname{algorithm}{Algorithm}{Algorithms}
\crefname{figure}{Fig.}{Figs.}
\crefname{table}{Table}{Tables}
\crefname{theorem}{Theorem}{Theorems}
\crefname{lemma}{Lemma}{Lemmas}
\crefname{corollary}{Corollary}{Corollaries}
\crefname{definition}{Definition}{Definitions}
\usepackage{booktabs}    
\usepackage{flushend} 

\setcopyright{ifaamas}  
\copyrightyear{2020} 
\acmYear{2020} 
\acmDOI{} 
\acmPrice{} 
\acmISBN{} 
\acmConference[AAMAS'20]{Proc.\@ of the 19th International Conference on Autonomous Agents and Multiagent Systems (AAMAS 2020)}{May 9--13, 2020}{Auckland, New Zealand}{B.~An, N.~Yorke-Smith, A.~El~Fallah~Seghrouchni, G.~Sukthankar (eds.)}  

\begin{document}

\title{Dueling Bandits: From Two-dueling to Multi-dueling}  


\author{Yihan Du}
\affiliation{%
	\institution{IIIS, Tsinghua University}
	\city{Beijing, China} 
}
\email{duyh18@mails.tsinghua.edu.cn}

\author{Siwei Wang}
\affiliation{%
	\institution{IIIS, Tsinghua University}
	\city{Beijing, China} 
}
\email{wangsw15@mails.tsinghua.edu.cn}

\author{Longbo Huang}
\affiliation{%
	\institution{IIIS, Tsinghua University}
	\city{Beijing, China} 
}
\email{longbohuang@mail.tsinghua.edu.cn}

\begin{abstract}  
	We study a general multi-dueling bandit problem, where an agent compares multiple options simultaneously and aims to minimize the regret due to selecting suboptimal arms. 
	This setting generalizes the traditional two-dueling bandit problem and  finds many real-world applications involving subjective feedback on multiple options. 
	We start with the two-dueling bandit setting and propose two efficient algorithms, DoublerBAI and MultiSBM-Feedback. 
	DoublerBAI provides a generic schema for translating known results on best arm identification algorithms to the dueling bandit problem, and achieves a regret bound of $O(\ln T)$. MultiSBM-Feedback not only has an optimal $O(\ln T)$ regret, but also reduces the constant factor by almost a half compared to benchmark results.  
	Then, we consider the  general multi-dueling case and develop an efficient algorithm MultiRUCB. 
	Using a novel finite-time regret analysis for the general multi-dueling bandit problem, we show that MultiRUCB also achieves an $O(\ln T)$ regret bound and the bound tightens as the capacity of  the comparison set  increases. 
	Based on both synthetic and real-world datasets, we empirically demonstrate that our algorithms outperform existing algorithms. 
\end{abstract}

%


\maketitle


\section{Introduction} 
The stochastic Multi-Armed Bandit (MAB) problem is a classic online learning problem and has been extensively studied \cite{thompson1933,UCB_auer2002,TS_agrawal2012analysis}. It has a wide range of applications such as clinical trials \cite{villar2015clinical_trials}, recommendation systems \cite{kohli2013recommendation_systems}, and online advertisement \cite{chakrabarti2009mortal_online_advertisement}. 
In the MAB problem,  an agent chooses one option from $K$ alternatives, often called ``arms,''  and observes a numerical reward at each time-step. The goal is to minimize the cumulative regret, defined as the expected difference between the actual reward collected and the offline optimal reward. 

The dueling bandits problem \cite{IF_JCSS2012} is an important variant of the MAB problem. In this problem, an agent chooses a pair of arms every time, but only observes the outcome of a noisy comparison between the two selected  arms. This setting is particularly useful in applications involving implicit or subjective (human) feedback, such as information retrieval \cite{hofmann2013_Information_Retrieval} and recommendation systems \cite{kohli2013recommendation_systems}. 

The original dueling bandit setting focuses on only comparing two arms at any time. 
%
In this work, we consider a general $K$-armed multi-dueling bandit problem which has been studied by \cite{Multi-dueling_bandits_and_their_application,Multi-dueling_Bandits_with_Dependent_Arms}, and propose a novel finite-time analysis for that. In this setting, an agent selects a subset of  arms with size at most $m$ ($2 \leq m \leq K$), and observes  pairwise dueling outcomes in the selected subset. 
The objective  is to minimize the regret, being the  advantage that the optimal arm has over the chosen  arms, cumulated up to $T$ plays. 

This multi-dueling bandit model can be used in many real-world applications. 
For example, in information retrieval, the emergence of numerous ranking algorithms (often called ``rankers,'' \emph{e.g.}, PageRank \cite{PageRank} and BM25 \cite{BM25}) necessitates efficient methods to evaluate these rankers.
Conventional online ranker evaluation methods often use interlearving comparison \cite{interleav2015_sigir,interleav2017_sigir,interleav2018_sigir}, which produces a combined result list of two rankers and translates the user clicks on this list  to  preference feedback. Recently, several multileaving methods \cite{brost2016multileaving_sigir,multileaving_2017,multileaving_2019} have been proposed, which permits multiple rankers to be compared at once and provides detailed feedback about how these rankers compare to each other, using less data than sequential interleaving comparisons. However, previous works did not address the key issue of how to select a subset of rankers for each comparison, in order to balance  between finding the potentially optimal ranker and presenting results of instantaneously good rankers to users, namely the exploration-exploitation trade-off. 
%
The multi-dueling bandit model, on the other hand, provides a principled way of selecting multiple rankers (``arms'') for each comparison with the objective of guaranteeing few results of poor rankers to be presented to users (regret minimization). 

Another application of the multi-dueling bandit model is the problem of online treatment decision in clinical trials. For instance, in motor function recovery, patients' motor responses to treatments are hard to quantify. 
Thus,  treatment performance is evaluated by clinicians via pairwise comparisons \cite{sui2014clinical}. Since  clinical trial is expensive and time-consuming, it is more efficient to compare multiple treatments simultaneously in a single trial rather than conducting sequential trials on treatment pairs. 
The clinicians often provide a ranking of patients'  recovery status in a trial, which can be transformed to all pairwise  feedbacks. The multi-dueling bandit model can efficiently handle this sequential decision making problem to  maximize treatment gains with lower economic costs.



%
Note that our algorithm and finite-time analysis for the multi-dueling bandit setting is not a trivial extension.
Indeed, if one naively extends algorithms for two-dueling bandits to multi-dueling bandits by repeatedly performing the  original strategies of selecting two arms, it is hard to simultaneously guarantee an efficient selection of comparing arms, a small overall regret, and that the regret improves as $m$ increases, three desired features of effective algorithms.  

To design efficient algorithms for our problem, we first revisit the original dueling bandit problem,
and propose two efficient algorithms, called DoublerBAI and MultiSBM-Feedback. 
Our algorithms build upon the Doubler and MultiSBM algorithms in \cite{Reducing_dueling_bandits}, which reduces the dueling bandits problem to the conventional stochastic MAB problem. 
DoublerBAI incorporates Best Arm Identification (BAI) algorithms to the dueling bandit problem, and improves the regret bound of Doubler from $O((\ln T)^2)$ to optimal $O(\ln T)$. 
MultiSBM-Feedback, on the other hand, not only has an optimal regret bound of $O(\ln T)$, but also reduces the constant factor of the logarithmic term by almost a half, compared to benchmark results. This regret bound is comparable with that of UCB \cite{UCB_auer2002} in a standard MAB problem in terms of both order and factor.
We then turn to the general formulation with comparing $m$ arms, and propose an efficient algorithm, called MultiRUCB. We prove that MultiRUCB achieves an $O(\ln T)$ regret, and the regret bound tightens as  the size of  the comparison set  $m$ increases, which cannot be achieved by directly applying existing two-dueling bandit solutions. This implies that given the ability of simultaneously comparing more arms, MultiRUCB efficiently exploits more information, and its performance boosts as such ability increases.


While there have been previous work on the  multi-dueling bandit problem \cite{Multi-dueling_bandits_and_their_application,Multi-dueling_Bandits_with_Dependent_Arms}, to the best of our knowledge, this is the first work to provide a finite-time regret analysis for the general multi-dueling bandit problem. Moreover, we conduct  experiments based on both the synthetic and real-world datasets \cite{Introducing_LETOR_Datasets}. The results demonstrate the  superior performance of our algorithms over existing benchmarks.

\section{Problem Setting}

We consider a general $K$-armed multi-dueling bandit problem, where an agent is given a set of $K$ arms, denoted by $\mathcal{X}:=\{x_1,x_2,$ $...,x_K\}$. 
At each time-step $t\in\{1, 2, ..., T\}$, the agent selects a subset $\mathcal{A}_t\subset\mathcal{X}$ for comparison, where the size of $\mathcal{A}_t$ is constrained by $|\mathcal{A}_t| \leq m$ ($2 \leq m \leq K$), 
and observes all pairwise dueling outcomes in $\mathcal{A}_t$. 
%
%
Specifically, dueling comparison works as follows \cite{Reducing_dueling_bandits}. 
Each arm $x_i \in \mathcal{X}$ has a latent utility distribution in $[0, 1]$ with expectation $\mu(x_i)$.  
Then, there is a link function $\phi: [0,1] \times [0,1] \mapsto [0,1]$, based on which the probability that arm $x_i$ beats arm $x_j$ is given by $p_{ij}=\phi \big(\mu(x_i), \mu(x_j) \big)$. 
The dueling outcome for arm $i$ and arm $j$ at every time is an independent Bernoulli random variable that takes value $1$, representing arm $i$ beats arm $j$, with probability $p_{ij}$. 

As in \cite{Reducing_dueling_bandits}, in this paper, we focus on the following linear link function:\footnote{We also extend our results to more general non-utility-based models \cite{IF_JCSS2012} in \cref{section_extension}, and show numerical results for the extended models in our experiments.}
\begin{eqnarray}
\phi \big(\mu(x_i), \mu(x_j) \big):= \frac{\mu(x_i)- \mu(x_j)+1}{2}. \nonumber 
\end{eqnarray} 
We also assume without loss of generality that $\mu(x_1)> \mu(x_2) \geq ... \geq \mu(x_K)$. 
We use $P:=[p_{ij}]$, whose ${ij}$-th entry is  the preference probability $p_{ij}$, to denote the $K \times K$ preference matrix .


%
%

For the multi-dueling bandit problem, the  expected cumulative regret up to time $T$ is defined to be: 
\begin{align}
\mathbb{E}[R_T]:=\sum \limits_{t=1}^{T} \sum \limits_{a \in \mathcal{A}_t} \frac{1}{|\mathcal{A}_t|}\Delta(x_1, a), \nonumber
\end{align}
where  $\Delta(x_i, x_j) :=p_{ij}-\frac{1}{2} \in [-\frac{1}{2}, \frac{1}{2}]$ is a  measure of the distinguishability between two arms. 
This  regret measures the average advantage that the best arm  has over the $|\mathcal{A}_t|$ arms  being chosen at each time-step $t$. This implies that an expected zero regret can be achieved if and only if $\mathcal{A}_t=\{x_1\}$. 
Note that when $m=2$, our problem becomes the original two-dueling bandit problem. For ease of notation, below we write $\Delta_{ij}$ for $\Delta(x_i, x_j)$  and $\Delta_{i}$ for $\Delta(x_1, x_i)$. 

Note that our multi-dueling bandit formulation is different from \cite{Multi-dueling_bandits_and_their_application}. 
In our setting, the algorithm can choose at most $m$ different arms  rather than  an arbitrary subset of $K$ arms at each time-step $t$. This scenario fits many practical applications better, as the number of arms being compared simultaneously is often constrained. While our multi-dueling bandit setting is the same to that in \cite{Multi-dueling_Bandits_with_Dependent_Arms}, we are the first to provide a finite-time regret analysis for this problem.

\section{Algorithms for Two-Dueling Bandits} 
We first start from the special case when $m=2$, i.e., the original two-dueling bandit problem,\footnote{When $m=2$, having $|\mathcal{A}_t|=1$ is equivalent to selecting $(a_0, a_0)$ (in this case $\mathcal{A}_t$ only contains a single arm $a_0$) in the original two-dueling bandit problem. Therefore, our setting reduces to the original two-dueling bandit problem when $m=2$.}   and propose two efficient algorithms  DoublerBAI and MultiSBM-Feedback for achieving an optimal regret. 
Our algorithms build upon the Doubler and MultiSBM algorithms in \cite{Reducing_dueling_bandits}. 

\subsection{DoublerBAI with Best Arm Identification Algorithms} \label{subsection_DoublerBAI}

\begin{algorithm}[!th]
	\caption{DoublerBAI}
	\label{algorithm_DoublerBAI}
	\LinesNumbered
	\KwIn{Exponentially growing sequence $\{T_i\}_{i \in \mathbb{N}}$, where $T_i=  \lfloor  a^{b^i} \rfloor \   (a, b>1)$}
	$S$ $\leftarrow$ new BAIM over $\mathcal{X}$\;
	Set the identified best arm $\hat{x}_{i}\  = \textup{NULL}$\ for all epoch $i \in \{0,1,...\}$\;
	Set the length of epoch $i$  $\tau_i=
	\left\{\begin{matrix}
	&T_0, &i=0
	\\ 
	&T_i-T_{i-1}, &i>0
	\end{matrix}\right.$\;
	\While{\textup{true}}
	{
		
		\If {$\hat{x}_{i-1} \neq \textup{NULL}$}
		{
			$\bar{x}_i \leftarrow\hat{x}_{i-1}$\;
		}
		\Else
		{
			Choose $\bar{x}_i$ randomly from $\mathcal{X}$\;
		}
		Reset($S$, $\delta_i=\frac{1}{\tau_{i+1}}$)\;
		\For{$j=1,...,\tau_i$}
		{
			\If {$\hat{x}_{i} \neq \textup{NULL}$}
			{
				$//\mathtt{exploit}$\\
				$y_t \leftarrow \hat{x}_{i}$\;
				Play ($\bar{x}_i, y_t$)\;
			}
			\Else  
			{
				$//\mathtt{explore}$\\
				$y_t \leftarrow \textup{Advance}(S)$\;
				Play ($\bar{x}_i, y_t$) and observe the binary dueling outcome $b_t$\;
				Feedback($S$, $b_t$)\;
				\If { \textup{StopTest}($S$)=\textup{true}}
				{
					$\hat{x}_i \leftarrow $ Return($S$)\;
					$\tau_i^{\textup{explore}} \leftarrow j$\;
				}
			}
			$t\leftarrow t+1 $\;
		}
		$i\leftarrow i+1 $\;
	}
\end{algorithm}

To present our algorithm, we define a generic Best Arm Identification Machine (BAIM) as a procedure which performs a $K$-armed BAI algorithm with an internal timer and memory, e.g., LUCB \cite{lucb}. 
A BAIM has five operations: Reset, Advance, Feedback, StopTest and Return. The Reset operation clears its state. The Advance operation decides the next arm to play. The Feedback operation updates its state with the observed information. The StopTest operation checks whether the internal BAI algorithm has terminated and the Return operation returns the identified best arm. 

With the BAIM procedure above, Algorithm \ref{algorithm_DoublerBAI} presents the formal definition of DoublerBAI. 
Generally speaking, we first divide the time horizon into exponentially growing epochs, motivated by the doubling trick \cite{doubling_trick_auer2010ucb,What_Doubling_Tricks}. 
Then, in each epoch $i$,  we fix one arm  $\bar{x}_i$ (the left arm)
of the played duel $(\bar{x}_i, y_t)$, and adaptively choose the other arm  $y_t$ (the right arm) using an  exploration-then-exploitation strategy. 

In the stage of exploration (Lines 16-22), we choose the right arm  $y_t$ according to the sample strategy provided by  $S$, the BAIM procedure, and feed back the dueling  outcome $b_t$ to  $S$. Thus, $S$ is actually estimating the probabilities of arms in $\mathcal{X}$ beating the fixed $\bar{x}_i$, and identifying the best arm. Once the internal BAI algorithm in $S$ terminates and returns the identified best arm $\hat{x}_i$ (Lines 20-22), we enter the  exploitation stage (Lines 12-14) and $y_t$ is chosen to be  $\hat{x}_i$. 
For the choice strategy of the left arm $\bar{x}_i$ (Lines 5-8), if $S$ terminates and returns a best arm in the previous epoch, i.e., $\hat{x}_{i-1} \neq \mathtt{NULL}$, then we set $\bar{x}_i$ to be the identified best arm $\hat{x}_{i-1}$ found in the previous epoch. Otherwise, we simply choose $\bar{x}_i$ randomly from $\mathcal{X}$. 

The key of DoublerBAI is to identify the best arm with high probability  in each epoch and  fix the left arm in the next epoch as the identified arm. The error probability of the BAIM in each epoch is set according to the length of the next epoch. This guarantees that the expected regret of the left  arm is a constant and the regret of the right arms is bounded by the internal regret of the BAIM. 

The following theorem provides the regret bound for DoublerBAI. 

\begin{restatable}{theorem}{thmdoubler} 
	\label{theorem1_DoublerBAI}
	Consider a K-armed utility-based two-dueling bandits game. Assume that the BAIM $S$ in DoublerBAI  has a sample complexity of $O(H \ln(\frac{H}{\delta}))$, where $S$ outputs the best arm with probability at least $1-\delta$. Given an exponentially growing sequence $\{T_i\}_{i \in \mathbb{N}}$ with parameters $a, b>1$, i.e., $T_i=  \lfloor  a^{b^i} \rfloor$, the expected regret of DoublerBAI is bounded by 
	\begin{align}
	\mathbb{E}[R_T] = & O((H \ln H)^b) +O(H \ln T)  \nonumber\\
	& + O(H \ln H \ln \ln T)+O(\ln \ln T), \nonumber
	\end{align}
	where $H:=\sum \limits_{i=2}^{K}\frac{1}{\Delta_i^2}$ is the problem complexity for a bandit instance.
\end{restatable}

\emph{Proof sketch.} 
(Please refer to  Section A of the supplementary material \cite{SupplementaryMaterial} for the full proof).

We first consider the regret incurred by  the right arm $y_t$. 
Let $B(\delta)$ denote the supremum of the expected regret of $S$ (the BAIM) to identify the best arm with probability at least $1-\delta$. 
In epoch $i$, after fixing the left arm $\bar{x}_i$, we see that $S$ is playing a standard BAI game in the stage of exploration by estimating the probabilities of arms in $\mathcal{X}$ to beat $\bar{x}_i$. Thus, in epoch $i$, the expected regret in  $S$ is 
$
\mathbb{E}[\sum \limits_{t=1} \limits^{\tau_i^{\textup{explore}}} \frac{\mu(x_1)-\mu(y_t)+1}{2} \ ] 
\leq B(\frac{1}{\tau_{i+1}})
$.
Specifically, according to the definition of regret for dueling bandits, we observe that the expected regret of the right arm $y_t$ in the stage of exploration, which  exactly equals to the left-hand side of the inequality, can be bounded by   $B(\frac{1}{\tau_{i+1}})$. 
Using the explore-then-exploit strategy, the expected regret of the right arm in epoch $i$ can be bounded by 
$ (1-\frac{1}{\tau_{i+1}})  B(\frac{1}{\tau_{i+1}}) + \frac{1}{\tau_{i+1}} O(\tau_i) $.
Taking a summation over all epochs (there are $O(\ln\ln T)$ epochs), we  obtain the main term $O(H \ln T)$ of the bound presented in  \cref{theorem1_DoublerBAI}.

Next, we consider the left arm. If the previous epoch returns an identified best arm $\hat{x}_{i-1}$ (with error probability at most  $\frac{1}{\tau_{i}}$), then  the left arm in epoch $i$  is fixed as $\bar{x}_i$=$\hat{x}_{i-1}$, which incurs expected regret of $\frac{1}{\tau_{i}}  \cdot O(\tau_i) + (1-\frac{1}{\tau_{i}}) \cdot 0$. Otherwise, the left arm $\bar{x}_i$ is chosen randomly, which incurs linear expected regret of $O(\tau_i)$. However, one can prove that the latter case only occurs in early short epochs, and the regret can be bounded by  $O((H \ln H)^b)$. $\Box$

\paragraph{Remark 1.} \cref{theorem1_DoublerBAI} suggests that our DoublerBAI improves the upper bound over its baseline, i.e., Doubler in \cite{Reducing_dueling_bandits}, from $O((\ln T)^2)$ (Theorem 3.1 in \cite{Reducing_dueling_bandits}) to $O(\ln T)$ by efficiently incorporating BAI algorithms. 
The upper bound of our DoublerBAI has an additional problem-dependent term $O((H \ln H)^b)$,  caused by not being able to identify best arms due to insufficient epochs length. 
Yet, by  setting $b$ close to $1$, $O((H \ln H)^b)$ becomes negligible for $T$ large enough, which is also efficient in practice.

\subsection{MultiSBM-Feedback with Multi-armed Bandit Algorithms}

We now consider the second algorithm, MultiSBM-Feedback, which  not only has an optimal regret bound of $O (\ln T)$, but also improves the constant factor of its baseline, i.e., MultiSBM in \cite{Reducing_dueling_bandits}.  

In MultiSBM-Feedback, we define a Singleton Bandit Machine (SBM) as a generic procedure representing a MAB algorithm with an internal timer and memory. In this work, we implement SBM with a variant of UCB \cite{UCB_auer2002} which satisfies the $\alpha$-robustness property defined in \cite{Reducing_dueling_bandits}. Below we restate this definition. 

\begin{definition}[$\alpha$-robustness]
	Let $T_i$ be the number of times a (sub-optimal) arm $x_i \in \mathcal{X}$ is played when running the policy $T$ rounds. A MAB policy is said to be \emph{$\alpha$-robust} when it has the following property: for all $s \geq 4(\alpha+4) \Delta_i^{-2} \ln(T)$, it holds that $\Pr[T_i > s] < \frac{2}{\alpha}(s/2)^{-\alpha}$.
\end{definition}

An SBM has four operations: Reset, Advance, Feedback and AdditionalFeedback. The first three operations are inherited from MultiSBM. 
The last AdditionalFeedback is newly added, and plays an important role in improving the regret. 
AdditionalFeedback receives an additional feedback sent from some arm  and updates  the SBM's internal state with the additional feedback. 

Algorithm \ref{algorithm_MultiSBM-Feedback} presents the procedure of MultiSBM-Feedback. Specifically, we operate $K$ different SBMs in parallel,  indexed by the $K$ elements in $\mathcal{X}$. 
SBM $S_x$ $(x \in \mathcal{X})$ performs an MAB algorithm via estimating  the probabilities of arms in $\mathcal{X}$ to beat arm $x$. At each time-step $t$, we choose the right arm $y_t$ of the duel $(x_t, y_t)$ according to the strategy provided by SBM $S_{x_t}$ and feed back the outcome $b_t^y$ ($b_t^y=1$ if $y_t$ wins against $x_t$, otherwise $b_t^y=0$) to $S_{x_t}$. If the  two arms are different, we invoke AdditionalFeedback to collect outcome $b_t^x = 1-b_t^y$  to $S_{y_t}$ (Lines 9-10). 
In the next time-step, the right arm $x_{t+1}$ is chosen to be $y_t$. In other words, the right arm in each time-step equals to the left arm in the next time-step.

The key of AdditionalFeedback is to exploit additional feedback from the perspective of $x_t$, to augment the information in $S_{y_t}$. This is because after one pull, the outcome of $x_t$ beating $y_t$ and that of $y_t$  beating $x_t$ can be respectively fed back to $S_{y_t}$ and $S_{x_t}$. 
Thus, $S_{y_t}$ receives an additional feedback from $x_t$ without pulling $x_t$, which helps $S_{y_t}$ augment its empirical observations on $x_t$. Note that in any SBM $S_x$, the empirical observations received from operations Feedback and AdditionalFeedback are independent. Thus,  the Chernoff-Hoeffding bound used in our theoretical analysis still holds. 

Algorithm \ref{algorithm_UCB_Feedback} presents the procedure of a SBM. $\rho_k$ denotes the number of times arm $x_k \in \mathcal{X}$ has been pulled. $s_k$ denotes the number of times this SBM receives additional feedback sent from arm $x_k$. The operation GetAdditionalFeedback is to obtain an additional feedback sent from  some left arm $x_t$ in Algorithm  \ref{algorithm_MultiSBM-Feedback}, which we label as $x_j$ in Algorithm  \ref{algorithm_UCB_Feedback}. 
If no additional feedback is sent to this SBM, GetAdditionalFeedback simply returns $\mathtt{NULL}$.  Every time before SBM pulls (advances) an arm, it  invokes GetAdditionalFeedback and  updates its empirical observations with the  additional feedback received from some arm $x_j$ (Lines 6-9). 

The following theorem bounds the expected regret of MultiSBM-Feedback.

\begin{algorithm}[t]
	\caption{MultiSBM-Feedback}
	\label{algorithm_MultiSBM-Feedback}
	\LinesNumbered
	For all $x \in \mathcal{X}$: $S_x \leftarrow$ new SBM over $\mathcal{X}$, Reset $(S_x)$\;
	$y_0 \leftarrow$ arbitrary element of $\mathcal{X}$\;
	$t \leftarrow 1$\;
	\While{\textup{true}}
	{
		$x_t \leftarrow y_{t-1}$\;
		$y_t \leftarrow$ Advance$(S_{x_t})$\;
		Play $(x_t, y_t)$, observe choice $b_t^y$\;
		Feedback$(S_{x_t}, b_t^y)$\;
		\If {$x_t \neq y_t$}
		{
			$b_t^x \leftarrow 1-b_t^y$, AdditionalFeedback$(S_{y_t}, b_t^x)$\;
		}
		$t \leftarrow t+1$\;
	}
\end{algorithm}

\begin{algorithm}[h]
	\caption{Implementation of SBM}
	\label{algorithm_UCB_Feedback}
	\LinesNumbered
	\KwIn{Confidence interval parameter $\alpha$}
	$\forall x_k \in \mathcal{X}$, set $\hat{\mu}_k=\infty$\;
	$\forall x_k \in \mathcal{X}$, set $\rho_k=0$\;
	$\forall x_k \in \mathcal{X}$, set $s_k=0$\;
	$t \leftarrow 1$\;
	\While{\textup{true}}
	{
		$b^{x_j}=$GetAdditionalFeedback$()$\;
		\If{ $b^{x_j} \neq$ \textup{NULL} }
		{
			$\hat{\mu}_{j}=\frac{\hat{\mu}_{j} \cdot (\rho_{j}+s_{j}) + b^{x_j} }{\rho_{j}+s_{j}+1} $\;
			$s_{j}=s_{j}+1$\;
		}
		Let $i$ be the  index maximizing $\hat{\mu}_i + \sqrt{ \frac{(\alpha+2)\ln t}{2(\rho_i+s_i)} }$;
		$//$  $\frac{x}{0}:=1$ $\mathtt{for}$ $\mathtt{any}$ $x$\\ 
		Play $x_i$, update $\hat{\mu}_i$, increment $\rho_i$ by 1\;
		$t \leftarrow t+1$\;
	}
\end{algorithm}

\begin{restatable}{theorem}{thmMultiSBMfeedback}
	\label{theorem2_MultiSBM-Feedback}
	Consider a K-armed utility-based two-dueling bandits game. 
	The expected regret of MultiSBM-Feedback, which implements an SBM defined in Algorithm  \ref{algorithm_UCB_Feedback},  is bounded by
	\begin{align*}
	& \mathbb{E}[R_T] 
	\leq \min \left \{ \sum \limits_{i > 1} \frac{(\alpha+2)\Delta_{max} }{\Delta_i^2} \ln T, \  \sum \limits_{i > 1} \frac{2(\alpha+2)}{\Delta_i} \ln T \right \}  
	\\
	& \! +\frac{(\alpha+8)\Delta_{max}}{2\alpha}K + \sum \limits_{j>1} \sum \limits_{i>1} O\Big( \frac{\alpha \Delta_{max}}{\Delta_j^2} \big( \ln\ln T + \ln K + \ln(\frac{1}{\Delta_i}) \big) \Big), \nonumber
	\end{align*}
	where  $\Delta_{max}:=\max \limits_{i>1}\Delta_{i}$ and the confidence interval parameter $\alpha=\max\{3,\frac{\ln K}{\ln \ln T} \}$.
\end{restatable}

\emph{Proof sketch.} 
(Please refer to  Section B of the supplementary material \cite{SupplementaryMaterial} for the full proof).

According to MultiSBM-Feedback (Algorithm  \ref{algorithm_MultiSBM-Feedback}), the right arm in each time-step equals to the left arm in the next time-step. Thus, in order to bound the total regret, it suffices to bound the number of times the right arm is suboptimal. Because the right arm is advanced by the SBM indexed by the left arm, we consider the regret from two parts, i.e., suboptimal right arms advanced by $S_{x_1}$ and  by  $S_x \  (x\neq x_1)$. 

We first analyze the latter part. Because the number of times a suboptimal arm $x\neq x_1$ being advanced in any SBM is $O(\ln T)$, according to the results of UCB \cite{UCB_auer2002}, the number of times  $x$ becomes the left arm  is  $O(K \ln T)$,  i.e., the internal timer of  $S_x \  (x\neq x_1)$ is order of $O(K \ln T)$. Thus, the number of times a suboptimal right arm advanced by $S_x$ is $O(\ln (K \ln T) )$. 

Next, we analyze the former part. By exploiting the additional feedbacks, we can prove that in $S_{x_1}$,  $\sum \limits_{i>1} \rho_i(t) =\sum   \limits_{i>1} s_i(t)$ for any internal time $t$. This is because every time   $S_{x_1}$ pulls a suboptimal arm ($\sum \limits_{i>1} \rho_i(t)$ increments by $1$), it must has received an additional feedback before ($\sum   \limits_{i>1} s_i(t)$ increments by $1$). Thus, we can prove an expected upper bound of $O(\ln T)$ for $\rho_i(t)+s_i(t) \  (i>1)$. Therefore, taking a summation over $i>1$,  we obtain a tighter upper bound of $\sum \limits_{i>1} \rho_i(t)$ compared to the original MultiSBM, where the order is still $O(\ln T)$, while the constant shrinks by a half. $\Box$

\paragraph{Remark 2.} \cref{theorem2_MultiSBM-Feedback} suggests that our MultiSBM-Feedback not only has an optimal regret bound of $O(\ln T)$, but also improves the constant factor of its benchmark result in MultiSBM.
This improvement is achieved by additionally exploiting the feedback from the duel. Moreover, the regret bound of MultiSBM is comparable to that of UCB \cite{UCB_auer2002} in a standard MAB setting in terms of both order and factor.  

\section{MultiRUCB for Multi-Dueling Bandits}

In this section, we consider the general case $2 \leq m \leq K$, where we can simultaneously compare multiple arms. We propose an efficient algorithm, called MultiRUCB, for the general multi-dueling bandit problem. We conduct a finite-time regret analysis and show that the regret of MultiRUCB is $O (\ln T)$  and tightens as the comparison set size $m$ increases. 
%
To the best of our knowledge, this is the first   finite-time regret analysis  for multi-dueling bandits. 

Algorithm  \ref{algorithm_MultiRUCB} presents the procedure of MultiRUCB. We define matrix $W_{K \times K}$  to record the empirical observations, whose ${ij}$-th entry denotes the number of times we observe $x_i$ beating $x_j$ ($x_i,x_j \in \mathcal{X}$).  
Motivated by \cite{Relative_upper_confidence_bound}, we also define the relative upper confidence bound matrix $U_{K \times K}$, whose ${ij}$-th entry  optimistically estimates the preference probability $p_{ij}$. We maintain a candidate set $\mathcal{C}$ which contains potential optimal arms and an empty or singleton set $\mathcal{B}$ which contains the hypothesized optimal arm. Note that the hypothesized optimal arm is removed from $\mathcal{B}$ once it   loses to another arm (Line $10$).  At each time-step $t$, we choose  the comparison set  $\mathcal{A}_t$ differently  according to the size of $\mathcal{C}$.  If $ \mathcal{C}=\varnothing $ (Lines $8-9$), we randomly choose $m$ different arms into $\mathcal{A}_t$  from $\mathcal{X}$, which is the trivial case and shown to occur infrequently in our analysis. 

Next we discuss  three non-trivial cases:
\begin{enumerate}[(a)]
	\item If $ |\mathcal{C}|=1 $, we are left with a single potential optimal arm $x_c$, which is hypothesized to be the optimal arm. We put the single arm  into $\mathcal{B}$ and  $\mathcal{A}_t$ (Lines $11-13$).
	\item If $ 1 <|\mathcal{C}| \leq m$, all potential optimal arms in $\mathcal{C}$ can be compared simultaneously. We simply put all of them into $\mathcal{A}_t$ (Lines $14-15$).
	\item If $|\mathcal{C}| > m$, we cannot put all potential optimal arms  into $\mathcal{A}_t$  at once. To choose $m$ different arms from $\mathcal{C}$,  if $\mathcal{B}$ is not empty, we give priority to the hypothesized optimal arm in $\mathcal{B}$ and choose the other arms uniformly at random. Otherwise, we uniformly and randomly choose $m$ different arms into $\mathcal{A}_t$  from $\mathcal{C}$ (Lines $16-22$).
\end{enumerate}

The key of MultiRUCB is to exploit as much information as possible from one pull to target $\mathcal{A}_t=\{x_1\}$.  $\mathcal{C}$ maintains a candidate pool for the potential optimal arms. When $\mathcal{C}$ contains multiple arms, which implies that the confidence region of some suboptimal arms are loose, we explore all of them simultaneously as possible. In the case this cannot be done, we wish to put optimal arm $x_1$ into $\mathcal{A}_t$. Thus, we give priority to the hypothesized optimal arm using the choice strategy define in Lines $18-22$. This is because $x_1$ is the most efficient arm to determine the sub-optimality of other arms.

The following theorem provides the  regret bound for MultiRUCB.

\begin{algorithm}[t]
	\caption{MultiRUCB}
	\label{algorithm_MultiRUCB}
	\LinesNumbered
	\KwIn{$\alpha > \frac{1}{2}$}
	$\mathbf{W}= [w_{ij}] \leftarrow \mathbf{0}_{K \times K}$\;
	$B \leftarrow \varnothing $\;
	\For{ $t=1,...,T$ }
	{
		$\mathbf{U}:= [u_{ij}] = \frac{\mathbf{W}}{\mathbf{W}+\mathbf{W}^T} +\sqrt{\frac{\alpha \ln t}{\mathbf{W}+\mathbf{W}^T}}$\; 
		$//$ $\mathtt{Element-wise}$ $\mathtt{operation;}$ $\frac{x}{0}:=1$ $\mathtt{for}$ $\mathtt{any}$ $x$\\
		$u_{ii} \leftarrow \frac{1}{2} $ for all $i \in \{1,...,K\}$\;
		$\mathcal{C} \leftarrow \{ x_c \ | \  u_{cj} \geq \frac{1}{2},  \ \forall j \in \{1,...,K\} \}$\;
		\If{$ \mathcal{C}=\varnothing $}
		{
			Randomly choose $m$ different arms for $\mathcal{A}_t$  from $\mathcal{X}$\; 
		}
		$\mathcal{B} \leftarrow \mathcal{B} \bigcap \mathcal{C}$\;
		\If{$ |\mathcal{C}|=1 $}
		{
			$\mathcal{B} \leftarrow \mathcal{C}$\;
			$\mathcal{A}_t \leftarrow \mathcal{C}$\;
		}
		\If{$ 1 <|\mathcal{C}| \leq m$}
		{
			$\mathcal{A}_t \leftarrow \mathcal{C}$\;
		}
		\If{$|\mathcal{C}| > m$}
		{	
			Choose $m$ different arms for $\mathcal{A}_t$  from $\mathcal{C}$  using the following strategy:\\
			\If{$\mathcal{B} = \varnothing $}
			{
				Uniformly choose $m$ different arms for $\mathcal{A}_t$  from $\mathcal{C}$\;
			}
			\Else
			{
				With probability of $\frac{1}{2}$, add $x_c \in \mathcal{B}$ into $\mathcal{A}_t$ and   uniformly add  $x_c \in \mathcal{C} \setminus \mathcal{B} $ into $\mathcal{A}_t$\;
				With probability of $\frac{1}{2}$,  uniformly choose $m$ different arms for $\mathcal{A}_t$  from $\mathcal{C} \setminus \mathcal{B}$\;
			}
		}
		Play $\mathcal{A}_t$ and observe all pairwise feedback in $\mathcal{A}_t$\;
		For any  pairwise feedback between   $x_j, x_k \in \mathcal{A}_t$, increment $w_{jk}$ or $w_{kj}$ depending on which arm wins\;
	}
\end{algorithm}

\begin{restatable}{theorem}{thmMultiRUCB}
	\label{theorem3_MultiRUCB}
	Consider a K-armed multi-dueling bandits game, where the number of comparing arms is at most $m$ at every time. Given $\alpha>1$, the expected regret of MultiRUCB  is bounded by
	\begin{align} 
	& \mathbb{E}[R_T] \leq  \left[ \left(\frac{2(4\alpha -1)K^2}{2\alpha -1} \right)^{\frac{1}{2\alpha -1}}  \frac{2\alpha -1}{\alpha -1} \right]\Delta_{max}  \nonumber
	\\
	& + \min \Bigg \{  D \Delta_{max}  \ln T,  \nonumber
	\\
	& \big(8+2D \ln 2D \big) \Delta_{max}  +  \frac{m+1}{m-1}\sum \limits_{i>1} \frac{4\alpha \Delta_{max}}{\Delta^2_i} \ln T \Bigg\},  \nonumber
	\end{align}
	where $D:=\sum \limits_{i>1} \frac{4\alpha}{\Delta^2_i}+ \sum \limits_{1<i<j} \frac{4\alpha}{C_m^2\Delta^2_{ij}}$  and $C_m^2:=\frac{m(m-1)}{2}$.
\end{restatable}

\emph{Proof sketch.} 
(Please refer to  Section C of the supplementary material \cite{SupplementaryMaterial} for the full proof).

We see that after $ C(\delta):= \left(\frac{(4\alpha -1)K^2}{(2\alpha -1) \delta} \right)^{\frac{1}{2\alpha -1}}$ time-steps, any preference probability $p_{ij}$ ($x_i,x_j \in \mathcal{X}$) will lie in its estimated confidence interval with probability at least $1-\delta$ (Lemma 1 in \cite{Relative_upper_confidence_bound}).  
Thus, with probability at least $1-\delta$,  after $ C(\delta)$ time-steps, $x_1$ exists in $\mathcal{C}$ ($u_{1i} \geq p_{1i} \geq \frac{1}{2},\  \forall i$). In order to bound the regret after $ C(\delta)$ time-steps, it suffices to bound the number of times  cases (b) or (c) occurs. 
For ease of notation, we define two subcases  (c-1) and   (c-2) of case (c). They respectively refer to the two situations where $x_1$ is added to $\mathcal{A}_t$ and not.

We first bound the sum of the number of times case (b) and case (c-1) occur.  Let $\widetilde{N}_{1i}(t) \  (i>1)$ denote the number of dueling outcomes between $x_1$ and $x_i$ we have observed, between time $C(\delta)+1$ and $t$. 
After $ C(\delta)$ time-steps, every time  case (b) occurs, we can observe at least one  outcome of duel between $x_1$ and some $x_i \  (i>1)$ ($\sum \limits_{i>1}  \widetilde{N}_{1i}(t)$ increments by 1). 
Every time  case (c-1) occurs, we can observe outcomes of  $m-1$  duels between  $x_1$ and $x_i \  (i>1)$ ($\sum \limits_{i>1}  \widetilde{N}_{1i}(t)$ increments by $m-1$).  
According to the definition of $\mathcal{C}$, we can prove $\widetilde{N}_{1i}(t) \leq    \frac{4\alpha}{\Delta^2_i} \ln t $. Thus, taking a summation over $i>1$,  the total number of times case (b) and case (c-1) occur, between time $C(\delta)+1$ and $t$,  is bounded by $  \sum \limits_{i>1}  \widetilde{N}_{1i}(t) \leq \sum \limits_{i>1} \frac{4\alpha}{\Delta^2_i} \ln t$. 

Next we bound the number of times case (c-2) occurs. We use $\widetilde{N}_{ij}(t) \  (1<i<j)$ to denote the number of dueling outcomes between $x_i$ and  $x_j$ we have observed between time $C(\delta)+1$ and $t$. After $ C(\delta)$ time-steps, every time  case (c-2) occurs, we can observe outcomes of  $C_m^2$  duels between  $x_i$ and $x_j$  ($ x_i, x_j \in \mathcal{X} \setminus \{ x_1\}$, $x_i \neq x_j$), i.e.,  $\sum \limits_{1<i<j} \widetilde{N}_{ij}(t)$ increments by $C_m^2$. According to the definition of $\mathcal{C}$, we can prove $\widetilde{N}_{ij}(t) \leq    \frac{4\alpha}{\Delta^2_{ij}} \ln t  $, implying $ \sum \limits_{1<i<j} \widetilde{N}_{ij}(t) \leq   \sum \limits_{1<i<j} \frac{4\alpha}{\Delta^2_{ij}} \ln t$. 
Since each occurrence of case (c-2) increments  $ \sum \limits_{1<i<j} \widetilde{N}_{ij}(t)$ by $C_m^2$, the  number of times case  (c-2) occurs between time $C(\delta)+1$ and $t$ is bounded by $  \sum \limits_{1<i<j} \frac{4\alpha}{C_m^2 \Delta^2_{ij}} \ln t$. Therefore, we obtain the term $D \Delta_{max}  \ln T$ in \cref{theorem3_MultiRUCB}.

Another term $\big(8+2D \ln 2D \big) \Delta_{max} +  \frac{m+1}{m-1} \sum \limits_{i>1} \frac{4\alpha \Delta_{max}}{\Delta^2_i} \ln T $ in \cref{theorem3_MultiRUCB} can be obtained by exploiting a geometric distribution with success probability $\frac{1}{2}$, following the procedures in \cite{Relative_upper_confidence_bound}. 
Specifically, we first need to investigate  when  $\mathcal{B}$ is set.  Define $\widehat{T}_{\delta}$ as the smallest time satisfying
$
\widehat{T}_{\delta}> C(\frac{\delta}{2}) +  D \ln \widehat{T}_{\delta}
$
, where $\widehat{T}_{\delta}$ is guaranteed to exist because the left side of the inequality grows linearly with $\widehat{T}_{\delta}$ and the right side grows logarithmically. It is easy to prove $\widehat{T}_{\delta} \leq 2C(\frac{\delta}{2})+2D \ln 2D$.
According to the definition of $\widehat{T}_{\delta}$,  with probability  at least $1-\frac{\delta}{2}$, there exists a time $T_{\delta} \in (C(\frac{\delta}{2}), \widehat{T}_{\delta} ]$ when case (a) occurs. This implies that with probability  at least $1-\frac{\delta}{2}$, $\mathcal{B}$ has been set as $\mathcal{B} =\{x_1\}$ from time $T_{\delta}$ on. 

Then, we know that from  time $T_{\delta}$ on, if MultiRUCB carries out case (c), case(c-1) will occur with probability of $\frac{1}{2}$.
Let  $\widehat{N}^{b}(t)$, $\widehat{N}^{c}_1(t)$ and  $\widehat{N}^{c}_2(t)$ denote the number of times case (b), (c-1) and (c-2) occur between time $T_{\delta}+1$ and $t$, respectively.
We also introduce  two sets of random variables, $\{\tau_0,\tau_1,\tau_2,...\}$ and $\{n_1,n_2,...\}$. Define $\tau_0:=T_{\delta}$ and $\tau_l$ as the $l^{th}$ time case (c-1) occurs after time $T_{\delta}$.  Define $n_l$ as the number of times case (c-2) occurs between $\tau_{l-1}$ and $\tau_l$.
Similar to the above analysis, we can prove that  with probability at least $1-\frac{\delta}{2}$, between time $T_{\delta}+1$ and $t$, case (c-1) occurs at most $L^{c}_1(t):= \sum \limits_{i>1} \frac{4\alpha}{(m-1)\Delta^2_i} \ln t $ times. 
Moreover, with probability at least $1-\frac{\delta}{2}$,  for any time  $t>T_{\delta}$, if case (c-1) has occurred  $L^{c}_1(t)$ times, all suboptimal arms $x_i \  (i>1)$ satisfy $u_{i1}<\frac{1}{2}$ and  case (c-2) cannot occur.
Thus, we can bound $\widehat{N}^{c}_2(t)$ by $\sum \limits_{l=1} \limits^{L^{c}_1(t)} n_l$. 
Since $n_l$ counts the number of times it takes for case (c) to produce one case (c-1), we can use the conclusion about geometric random variables to bound $\sum \limits_{l=1} \limits^{L^{c}_1(t)} n_l$. Therefore, we have that with probability at least $1-\delta$, $\forall t> T_{\delta}$,
$
\widehat{N}^{c}_2(t) \leq \sum \limits_{l=1} \limits^{L^{c}_1(t)} n_l  \leq 2\sum \limits_{i>1} \frac{4\alpha}{(m-1)\Delta^2_i} \ln t +4 \ln  \frac{2}{\delta}
$.
Taking summation over $T_{\delta}$, $\widehat{N}^{b}(t)$, $\widehat{N}^{c}_1(t)$  and  $\widehat{N}^{c}_2(t)$, we obtain the term $\big(8+2D \ln 2D \big) \Delta_{max} +  \frac{m+1}{m-1} \sum \limits_{i>1} \frac{4\alpha \Delta_{max}}{\Delta^2_i} \ln T $ in \cref{theorem3_MultiRUCB}.

%
At last, integrating the confidence term with respect to $\delta$, we obtain the expected regret bound in \cref{theorem3_MultiRUCB}. $\Box$

\paragraph{Remark 3.} \cref{theorem3_MultiRUCB} suggests that  compared to the two-dueling bandit solutions,  MultiRUCB  has the same $O(\ln T)$ regret. 
However, by exploiting more information from one pull, the regret bound of MultiRUCB tightens as the comparison set size $m$ increases, which is unachievable through only repeating existing two-dueling bandit solutions.
This  implies that our extension of the algorithm and finite-time analysis from two-dueling to multi-dueling is non-trivial and useful. Moreover, to the best of our knowledge, MultiRUCB is the first algorithm providing a finite-time regret analysis  for multi-dueling bandits. 

\begin{figure*}[!h]
	\centering
	\includegraphics[width=0.55\textwidth]{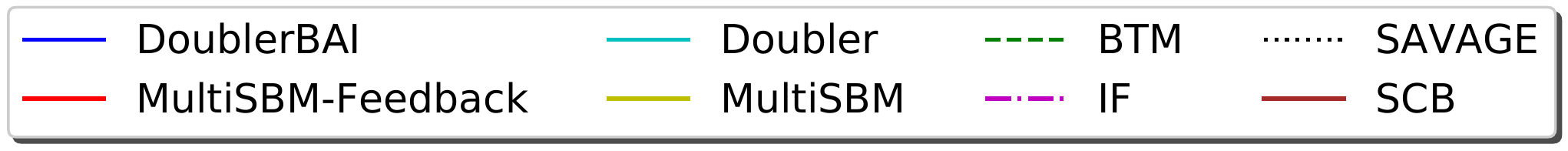}
	
	\subfigure[Synthetic, $K=48$, linear link function]{
		\includegraphics[height=0.22\textwidth]{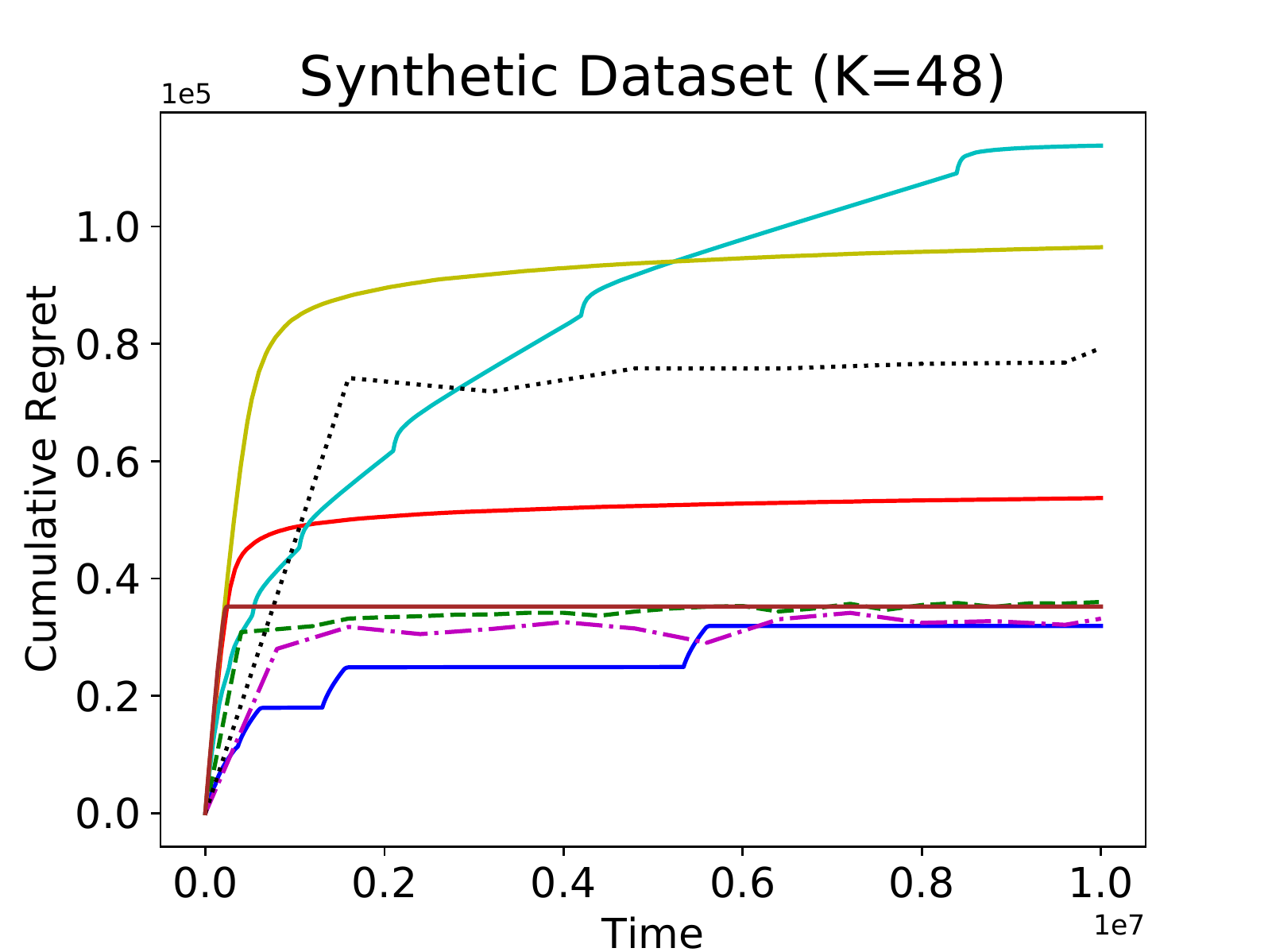}}
	\quad
	\subfigure[Synthetic, $K=48$, natural link function]{
		\includegraphics[height=0.22\textwidth]{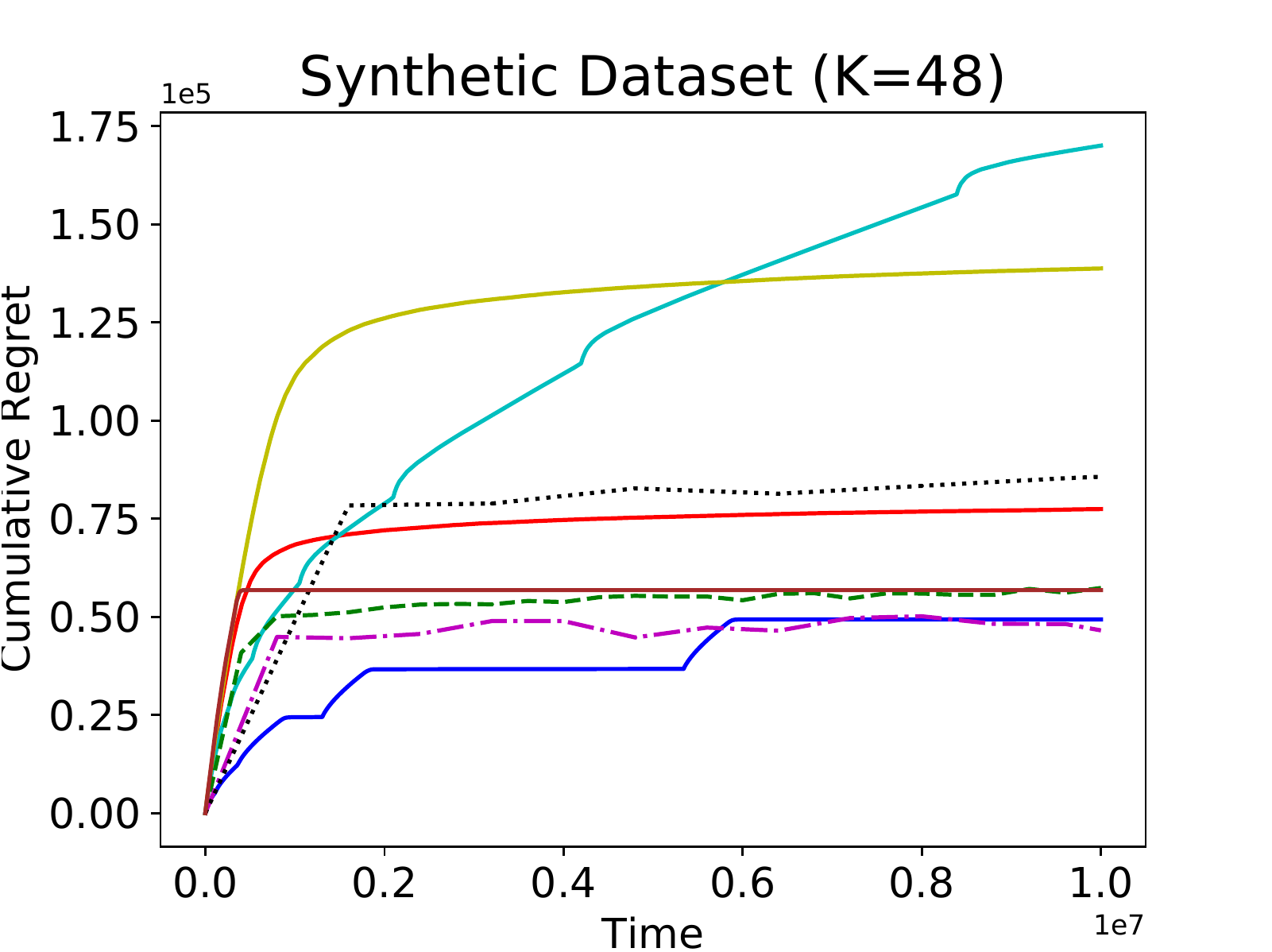}}
	\quad
	\subfigure[Synthetic, $K=48$, logit link function]{
		\includegraphics[height=0.22\textwidth]{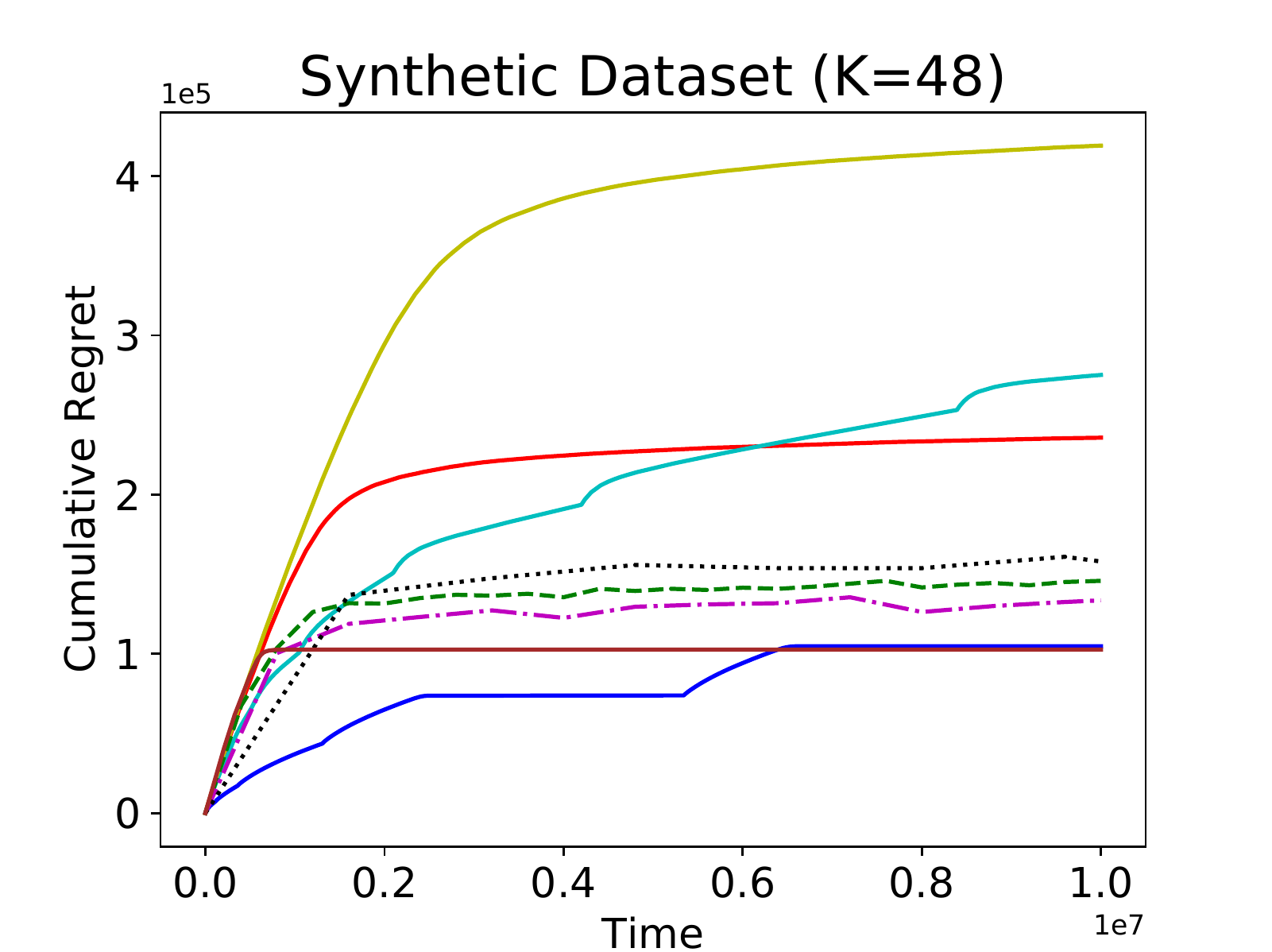}}
	
	\subfigure[Synthetic, $K=72$, linear link function]{
		\includegraphics[height=0.22\textwidth]{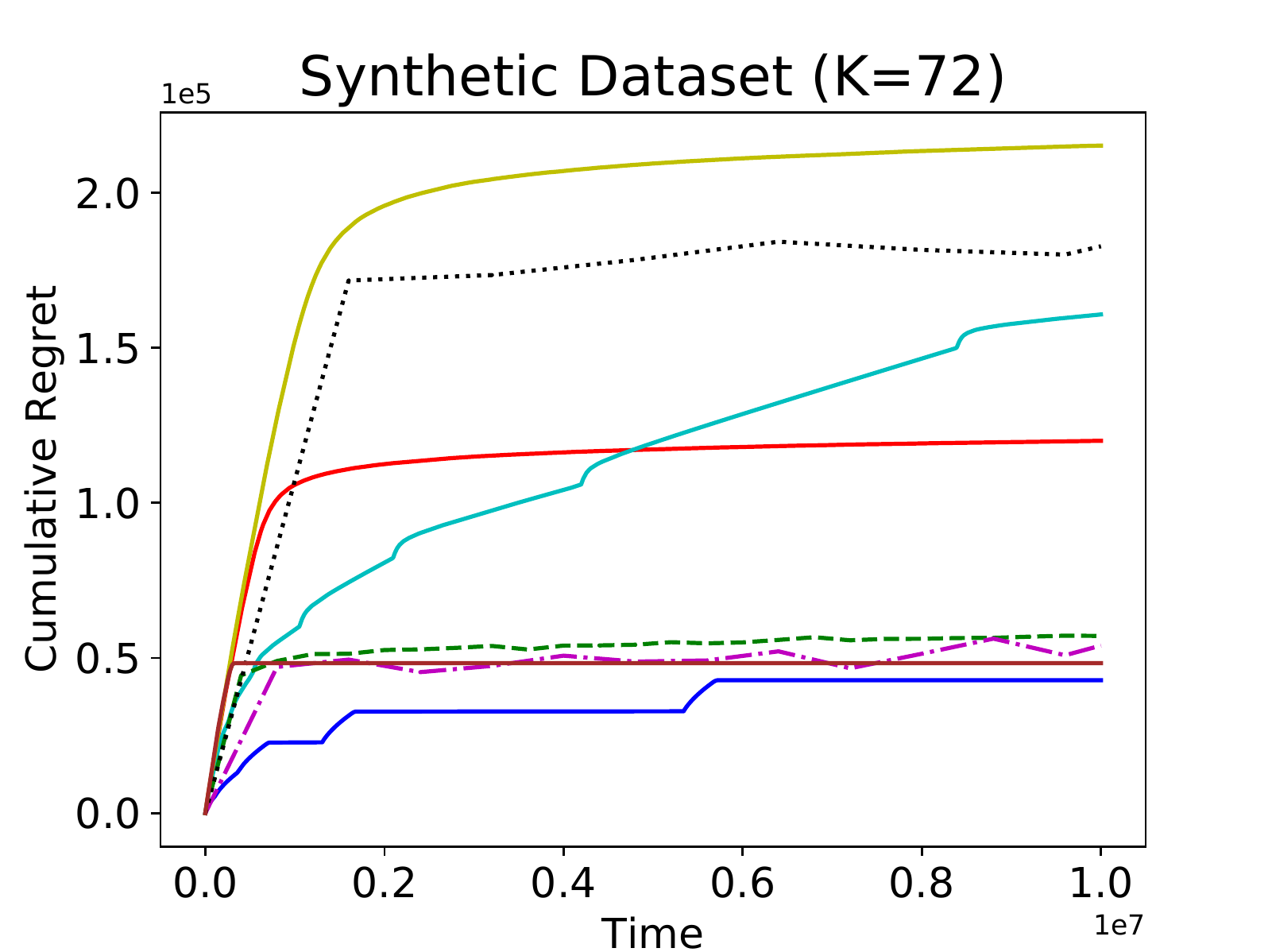}}
	\quad
	\subfigure[Synthetic, $K=72$, natural link function]{
		\includegraphics[height=0.22\textwidth]{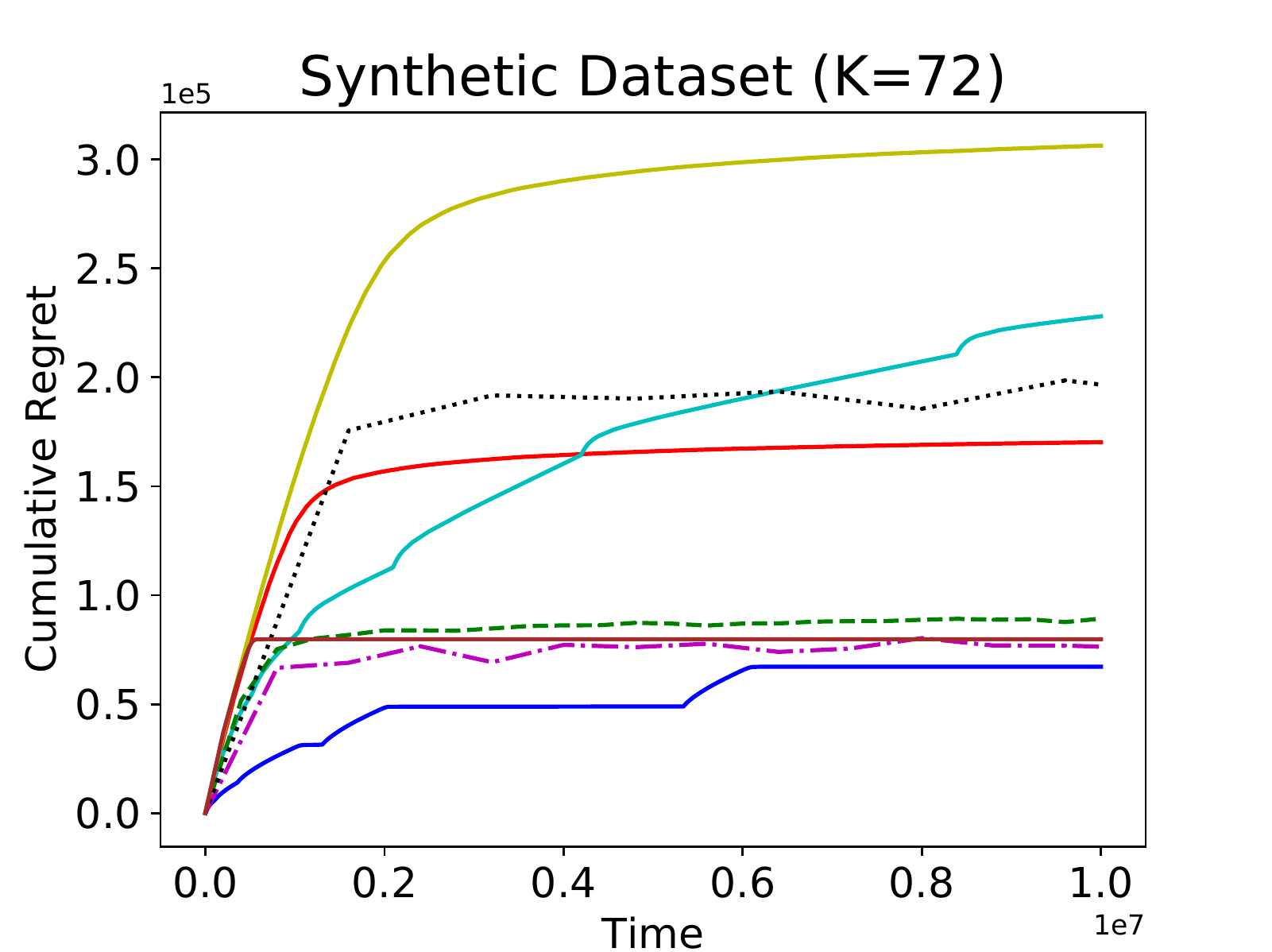}}
	\quad
	\subfigure[MSLR, $K=24$]{
		\includegraphics[height=0.22\textwidth]{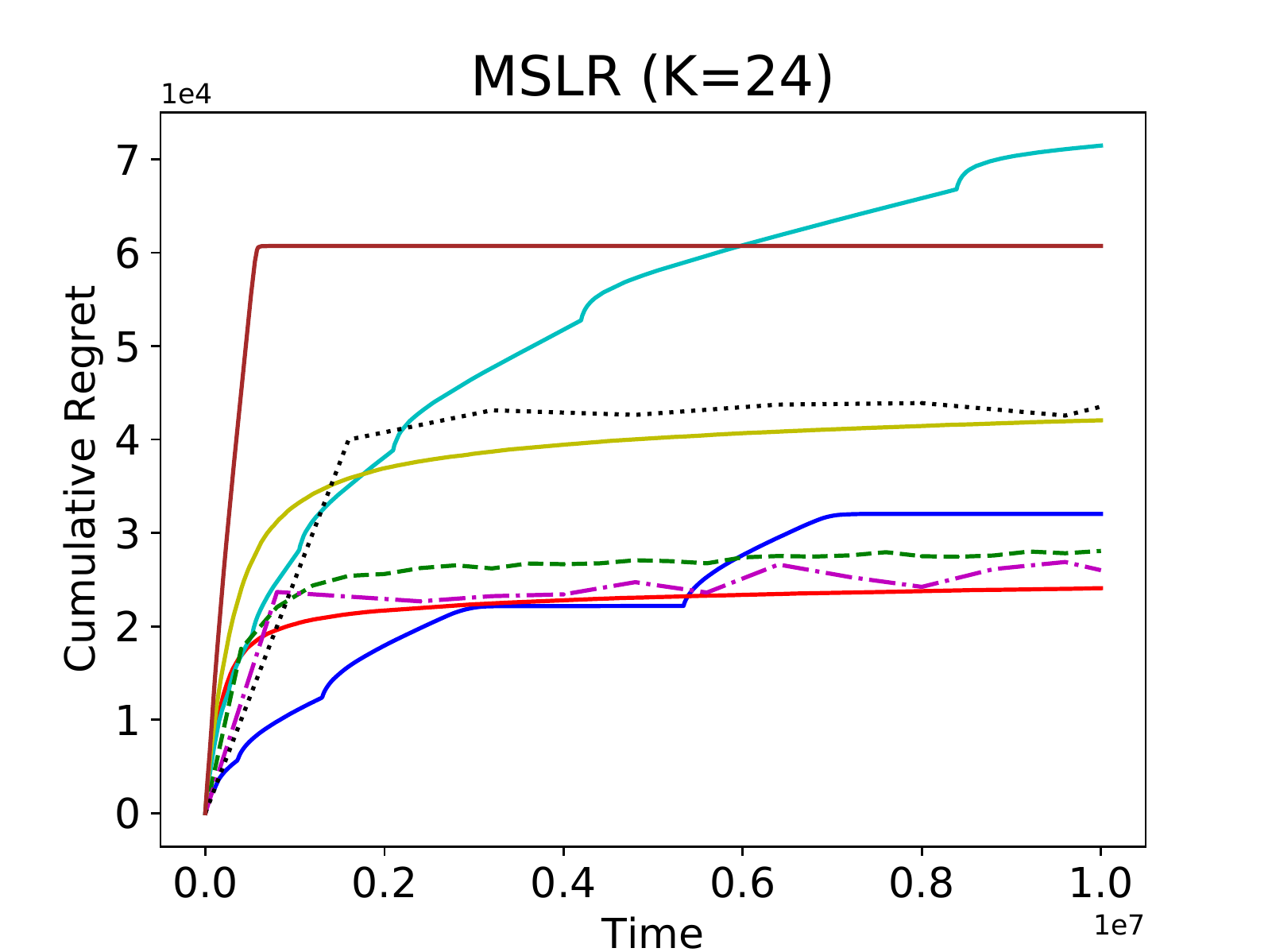}}
	
	\caption{Regret results of two-dueling bandits on the synthetic (a-e) and MSLR (f) dataset. }
	\label{fig_two-dueling}
\end{figure*}

\section{Extension of the Link Function}
\label{section_extension}
Our analysis of DoublerBAI and MultiSBM-Feedback assumes the linear link function $\phi \big(\mu(x_i), \mu(x_j) \big):= \frac{\mu(x_i)- \mu(x_j)+1}{2}$.  In this section, we generalize the linear link function to more general non-utility-based models in \cite{IF_JCSS2012}.

It can be verified that our analysis still holds when $\Delta(\cdot, \cdot)$ satisfies the following property:

\paragraph{Property 1.} For some $\gamma>0$ and any two arms $x_i,x_j \in \mathcal{X}$, 
\begin{eqnarray*}
	\Delta (x_1, x_i) \leq \gamma (\Delta(x_1,x_j)-\Delta(x_i,x_j)).
\end{eqnarray*}

This property holds for a wide family of $\Delta(\cdot, \cdot)$. 
The main idea is that our analysis  holds if  the regret in the dueling bandits problem can be bounded by the regret seen by the BAIM (in Doubler) and SBM (in MultiSBM-Feedback)  with some positive  $\gamma$.
The effect of $\gamma$ on the regret bound of DoublerBAI and MultiSBM-Feedback is shown in the following corollaries:

\begin{corollary}
	Consider a K-armed two-dueling bandits game, in which  $\Delta(\cdot, \cdot)$ satisfies Property 1 with parameter $\gamma$. 
	Assume that the BAIM $S$ in Line 1 of DoublerBAI  has a sample complexity of $O(H \ln(\frac{H}{\delta}))$, where $S$ outputs the best arm with probability at least $1-\delta$. Given an exponentially growing sequence $\{T_i\}_{i \in \mathbb{N}}$ of parameters $a,b>1$ (i.e., $T_i=  \lfloor  a^{b^i} \rfloor$), the expected regret of DoublerBAI is bounded by 
	\begin{align}
	\mathbb{E}[R_T] = & O((H \ln H)^b) +O(H \ln T)  \nonumber\\
	& + O(H \ln H \ln \ln T)+O(\ln \ln T), \nonumber
	\end{align}
	where $H:=\sum \limits_{i=2}^{K}\frac{1}{\Delta_i^2}$ is the problem complexity for a bandit instance.
\end{corollary}

\begin{corollary}
	Consider a K-armed two-dueling bandits game, in which  $\Delta(\cdot, \cdot)$ satisfies Property 1 with parameter $\gamma$. 
	The expected regret of MultiSBM-Feedback, which implements an SBM defined in  Algorithm \ref{algorithm_UCB_Feedback},  is bounded by
	\begin{align*}
	& \mathbb{E}[R_T] 
	\leq  \min \left \{ \sum \limits_{i > 1} \frac{(\alpha+2)\Delta_{max} }{\Delta_i^2} \ln T, \  \sum \limits_{i > 1} \frac{2(\alpha+2)}{\Delta_i} \ln T \right \} +
	\nonumber
	\\
	& \! \frac{(\alpha+8)\Delta_{max}}{2\alpha}K  +\sum \limits_{j>1} \sum \limits_{i>1} O\Big( \frac{\gamma \alpha \Delta_{max}}{\Delta_j^2} \big( \ln\ln T + \ln K + \ln(\frac{1}{\Delta_i}) \big) \Big),
	\nonumber
	\end{align*}
	where  $\Delta_{max}:=\max \limits_{i>1}\Delta_{i}$ and the confidence interval parameter   $\alpha=\max\{3,\frac{\ln K}{\ln \ln T} \}$.
\end{corollary}

Note that $\gamma$ does not affect the regret bound of DoublerBAI and the main  term in the regret bound of MultiSBM-Feedback. This is because when fixing $x_j=x_1$ in Property 1, $\gamma$ just vanishes and does not affect our analysis.

\begin{figure*}[!h]
	\centering
	\hspace*{-2em}
	\subfigure[Synthetic, $K=48$, $m=8$]{
		\includegraphics[height=0.24\textwidth]{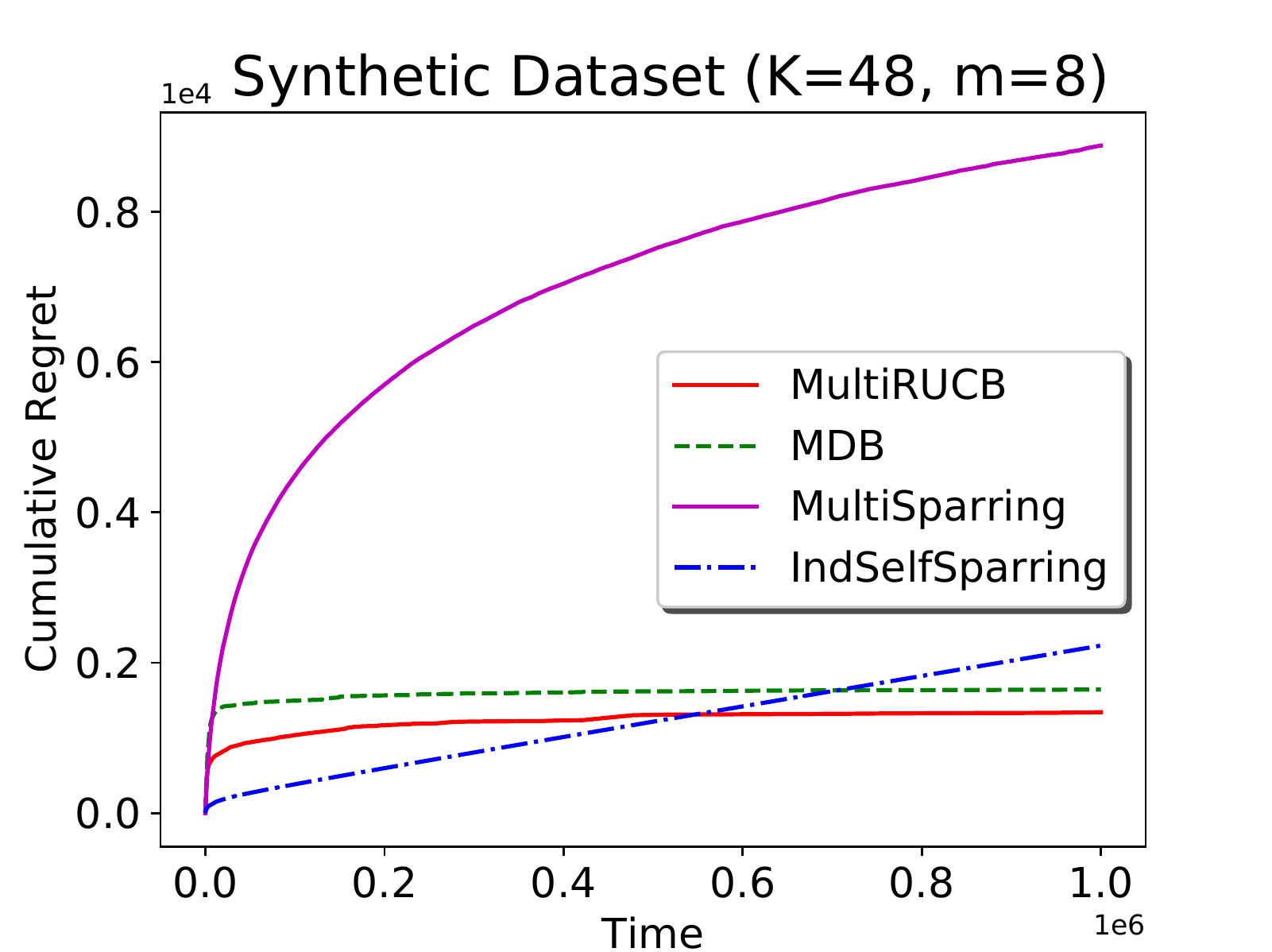}}
	\hspace*{2em}
	\subfigure[Synthetic, $K=48$, $m=16$]{
		\includegraphics[height=0.24\textwidth]{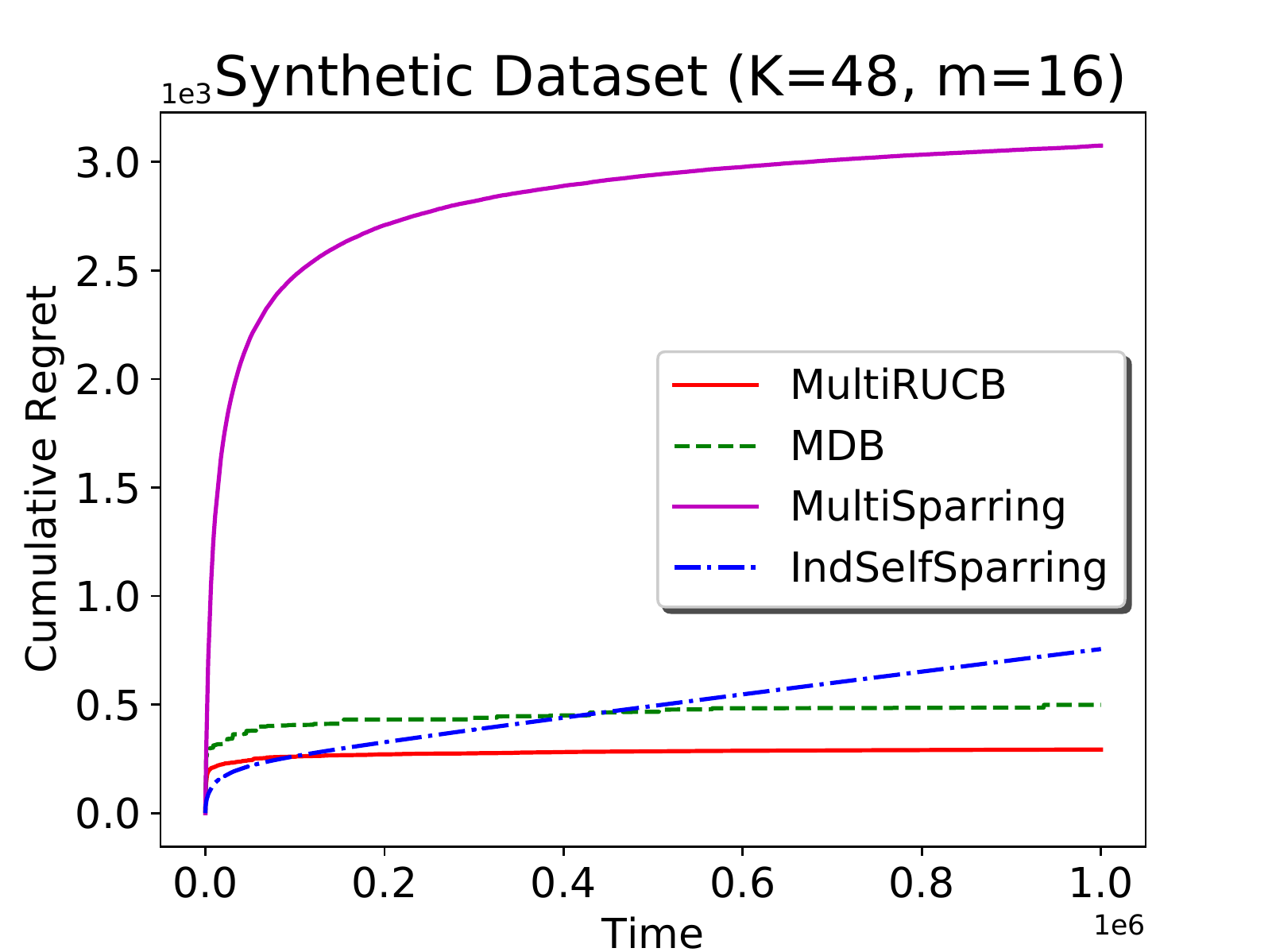}}
	
	\hspace*{-2em}
	\subfigure[MSLR, $K=24$, $m=8$]{
		\includegraphics[height=0.24\textwidth]{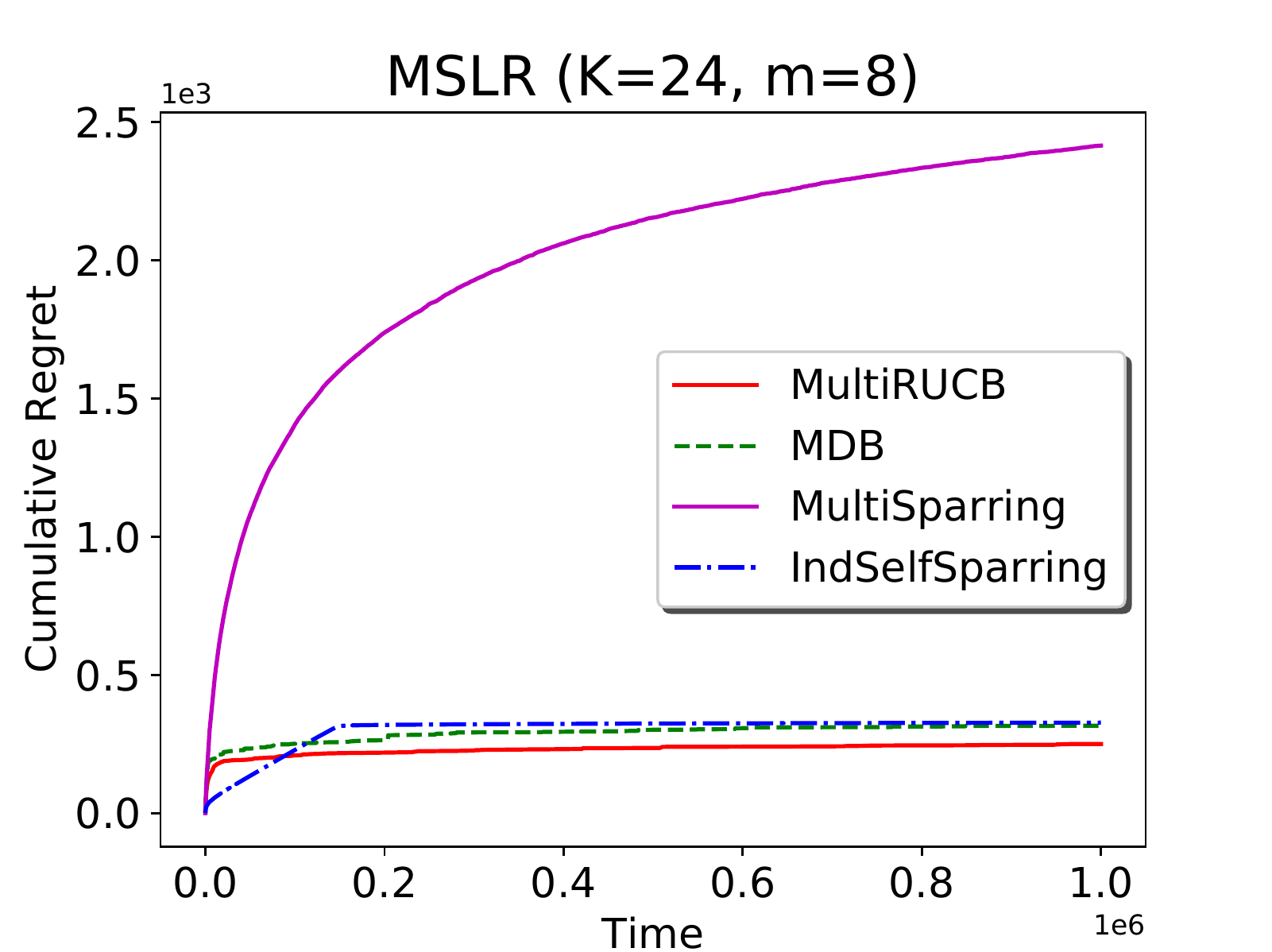}}
	\hspace*{2em}
	\subfigure[MSLR, $K=24$, $m=16$]{
		\includegraphics[height=0.24\textwidth]{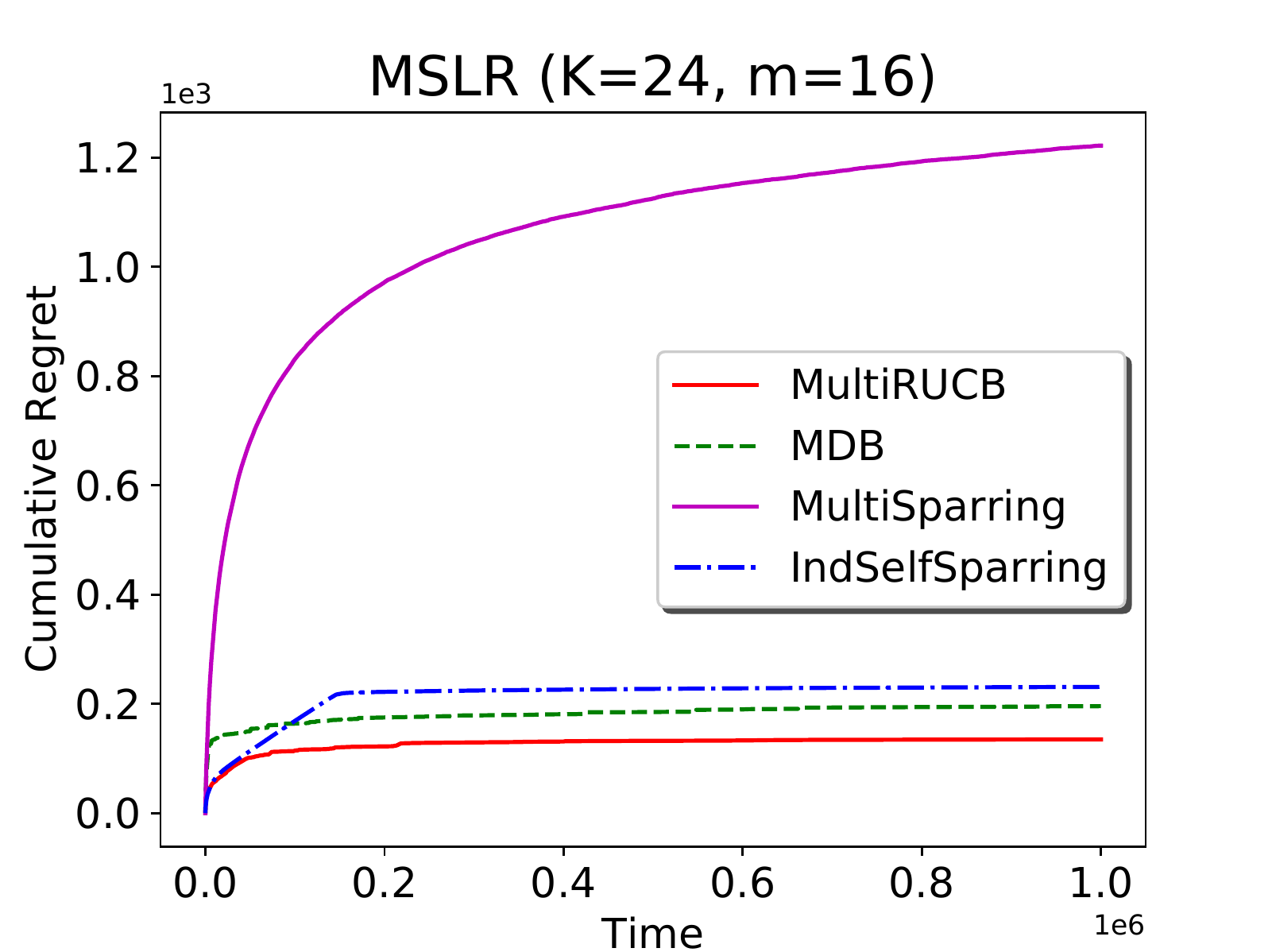}}
	\caption{Regret results of multi-dueling bandits on the synthetic (a-b) and MSLR (c-d) dataset. The results are obtained using the linear link function.}
	\label{fig_multi-dueling}
\end{figure*}

\section{Experiments} \label{section_experiments}

We conduct  experiments for two-dueling bandits and multi-dueling bandits on both the synthetic and real-world datasets.
In our synthetic datasets, the expected utilities of $K$ arms are set as $\mu(x_1)=0.8$ and  $\mu(x_2), ..., \mu(x_{K})$ forming a geometric sequence with  $\mu(x_2)=0.7, \mu(x_K)=0.2$. Moreover, besides the linear link function, we also conduct experiments for two additional link functions, natural  and logit, which are  respectively defined as follows:
\begin{align}
&\phi_{\mathtt{natural}} \big(\mu(x_i), \mu(x_j) \big):= \frac{\mu(x_i)}{\mu(x_i) + \mu(x_j)}, \nonumber
\\
&\phi_{\mathtt{logit}} \big(\mu(x_i), \mu(x_j) \big):= \frac{1}{1+\exp \big( \mu(x_j)-\mu(x_i) \big) }. \nonumber
\end{align}
For the real-world dataset, we use the Microsoft Learning to Rank (MSLR) dataset \cite{Introducing_LETOR_Datasets} in information  retrieval, which contains query-document pairs labeled with relevance scores. Our  setup follows that of  \cite{Copeland_dueling_bandits}, which estimates a preference matrix for $136$ rankers. Each ranker can be regarded as an arm in  dueling bandits since it is a function mapping a query of the user to a document ranking. We use a  submatrix of $24$ rankers selected from the full preference matrix. We remark here that the choice of $K=24$ here is made to ensure a total order of the chosen arms (since their relations are obtained from a preference matrix), such that the existence of an optimal arm is guaranteed. 
The presented results are averaged over 50 independent runs for each algorithm.

\subsection{Two-dueling Bandits Experiments}
For the special case of our general setting, i.e., two-dueling bandits, we compare  DoublerBAI and MultiSBM-Feedback with their baselines Doubler and MultiSBM \cite{Reducing_dueling_bandits}, and other state-of-the-art algorithms including IF \cite{IF_JCSS2012}, BTM \cite{btm}, SAVAGE \cite{savage} and SCB in \cite{Copeland_dueling_bandits}. For DoublerBAI, we choose the  LUCB algorithm in  \cite{lucb} as the BAIM, and set parameters $a=10, \  b=1.1$. For the finite-horizon algorithms, IF, BTM and SAVAGE, we obtain each point of their regret curves by resetting the horizon to the corresponding time value.
As shown in \cref{fig_two-dueling}, DoublerBAI and MultiSBM-Feedback not only achieve significant improvements  over their baselines, Doubler and MultiSBM, but also outperform the other state-of-the-art algorithms. In particular,   compared to MultiSBM, MultiSBM-Feedback reduces the regret  by approximately a half, which matches our theoretical analysis.

\subsection{Multi-dueling Bandits Experiments} \label{section_multiexp}
For the general multi-dueling bandit setting, we compare MultiRUCB with three state-of-the-art algorithms including MDB \cite{Multi-dueling_bandits_and_their_application}, IndSelfSparring \cite{Multi-dueling_Bandits_with_Dependent_Arms} and MultiSparring (the multi-dueling extension of Sparring \cite{Reducing_dueling_bandits}).  \cref{fig_multi-dueling} plots the average cumulative regrets for 50 independent runs in the cases $m=8$ and $m=16$\footnote{The results are similar for other $m$ values.}.  In addition, the variances of cumulative regrets at the $10^6$ timestep corresponding to \cref{fig_multi-dueling} (a-d) are also presented in Section D of the supplementary material \cite{SupplementaryMaterial} due to the space limit.
The experimental results show that our MultiRUCB not only achieves the best regret performance, but also ensures the smallest variances among all the compared algorithms on both the synthetic and MSLR dataset. This demonstrates the  superiority of  MultiRUCB in practice, compared to existing algorithms for multi-dueling bandits. 
Moreover, among all the compared algorithms, MultiRUCB is the only algorithm possessing a finite-time analysis.


\section{Conclusion}
In this work, we study a general multi-dueling bandit problem, which has extensive real-world applications involving simultaneous duels of multiple options. 
For the special case of our setting, two-dueling bandits, we propose two efficient algorithms DoublerBAI and MultiSBM-Feedback, both achieving $O(\ln T)$ regret and outperforming existing algorithms. 
For the general  multi-dueling bandits, we propose MultiRUCB and provide the first finite-time analysis for the problem. We prove that MultiRUCB achieves an  $O(\ln T)$ regret. We also show that its regret improves as  the capacity of  the comparison set  increases. 
%
Our experimental results based on both synthetic and real-world datasets demonstrate the performance superiority of our algorithms, compared to other state-of-the-art algorithms.


\section*{Acknowledgments}
The work is supported in part by the National Natural Science Foundation of China Grant 61672316, the Zhongguancun Haihua Institute for Frontier Information Technology and the Turing AI Institute of Nanjing.


\clearpage
\bibliographystyle{ACM-Reference-Format}
\balance
\bibliography{reference_aamas2020_camready} 


\begin{thebibliography}{30}


\ifx \showCODEN    \undefined \def \showCODEN     #1{\unskip}     \fi
\ifx \showDOI      \undefined \def \showDOI       #1{#1}\fi
\ifx \showISBNx    \undefined \def \showISBNx     #1{\unskip}     \fi
\ifx \showISBNxiii \undefined \def \showISBNxiii  #1{\unskip}     \fi
\ifx \showISSN     \undefined \def \showISSN      #1{\unskip}     \fi
\ifx \showLCCN     \undefined \def \showLCCN      #1{\unskip}     \fi
\ifx \shownote     \undefined \def \shownote      #1{#1}          \fi
\ifx \showarticletitle \undefined \def \showarticletitle #1{#1}   \fi
\ifx \showURL      \undefined \def \showURL       {\relax}        \fi
\providecommand\bibfield[2]{#2}
\providecommand\bibinfo[2]{#2}
\providecommand\natexlab[1]{#1}
\providecommand\showeprint[2][]{arXiv:#2}

\bibitem[\protect\citeauthoryear{??}{Sup}{sion}]%
        {SupplementaryMaterial}
 \bibinfo{year}{AAMAS2020 Submission}\natexlab{}.
\newblock \showarticletitle{Supplementary Material}.
  \bibinfo{howpublished}{\url{https://www.dropbox.com/s/nzscfpdsj3ne099/supp_camready_0215.pdf?dl=0}}.
\newblock


\bibitem[\protect\citeauthoryear{Agrawal and Goyal}{Agrawal and Goyal}{2012}]%
        {TS_agrawal2012analysis}
\bibfield{author}{\bibinfo{person}{Shipra Agrawal} {and} \bibinfo{person}{Navin
  Goyal}.} \bibinfo{year}{2012}\natexlab{}.
\newblock \showarticletitle{Analysis of thompson sampling for the multi-armed
  bandit problem}. In \bibinfo{booktitle}{\emph{Conference on Learning
  Theory}}. \bibinfo{pages}{39--1}.
\newblock


\bibitem[\protect\citeauthoryear{Ailon, Karnin, and Joachims}{Ailon
  et~al\mbox{.}}{2014}]%
        {Reducing_dueling_bandits}
\bibfield{author}{\bibinfo{person}{Nir Ailon}, \bibinfo{person}{Zohar Karnin},
  {and} \bibinfo{person}{Thorsten Joachims}.} \bibinfo{year}{2014}\natexlab{}.
\newblock \showarticletitle{Reducing dueling bandits to cardinal bandits}. In
  \bibinfo{booktitle}{\emph{Proceedings of the 31st International Conference on
  Machine Learning}}. \bibinfo{pages}{856--864}.
\newblock


\bibitem[\protect\citeauthoryear{Auer, Cesa-Bianchi, and Fischer}{Auer
  et~al\mbox{.}}{2002}]%
        {UCB_auer2002}
\bibfield{author}{\bibinfo{person}{Peter Auer}, \bibinfo{person}{Nicolo
  Cesa-Bianchi}, {and} \bibinfo{person}{Paul Fischer}.}
  \bibinfo{year}{2002}\natexlab{}.
\newblock \showarticletitle{Finite-time analysis of the multiarmed bandit
  problem}.
\newblock \bibinfo{journal}{\emph{Machine learning}} \bibinfo{volume}{47},
  \bibinfo{number}{2-3} (\bibinfo{year}{2002}), \bibinfo{pages}{235--256}.
\newblock


\bibitem[\protect\citeauthoryear{Auer and Ortner}{Auer and Ortner}{2010}]%
        {doubling_trick_auer2010ucb}
\bibfield{author}{\bibinfo{person}{Peter Auer} {and} \bibinfo{person}{Ronald
  Ortner}.} \bibinfo{year}{2010}\natexlab{}.
\newblock \showarticletitle{UCB revisited: Improved regret bounds for the
  stochastic multi-armed bandit problem}.
\newblock \bibinfo{journal}{\emph{Periodica Mathematica Hungarica}}
  \bibinfo{volume}{61}, \bibinfo{number}{1-2} (\bibinfo{year}{2010}),
  \bibinfo{pages}{55--65}.
\newblock


\bibitem[\protect\citeauthoryear{Besson and Kaufmann}{Besson and
  Kaufmann}{2018}]%
        {What_Doubling_Tricks}
\bibfield{author}{\bibinfo{person}{Lilian Besson} {and} \bibinfo{person}{Emilie
  Kaufmann}.} \bibinfo{year}{2018}\natexlab{}.
\newblock \showarticletitle{What Doubling Tricks Can and Can't Do for
  Multi-Armed Bandits}.
\newblock \bibinfo{journal}{\emph{arXiv preprint arXiv:1803.06971}}
  (\bibinfo{year}{2018}).
\newblock


\bibitem[\protect\citeauthoryear{Brost, Cox, Seldin, and Lioma}{Brost
  et~al\mbox{.}}{2016a}]%
        {brost2016multileaving_sigir}
\bibfield{author}{\bibinfo{person}{Brian Brost}, \bibinfo{person}{Ingemar~J
  Cox}, \bibinfo{person}{Yevgeny Seldin}, {and} \bibinfo{person}{Christina
  Lioma}.} \bibinfo{year}{2016}\natexlab{a}.
\newblock \showarticletitle{An improved multileaving algorithm for online
  ranker evaluation}. In \bibinfo{booktitle}{\emph{Proceedings of the 39th
  International ACM SIGIR Conference on Research and Development in Information
  Retrieval}}. ACM, \bibinfo{pages}{745--748}.
\newblock


\bibitem[\protect\citeauthoryear{Brost, Seldin, Cox, and Lioma}{Brost
  et~al\mbox{.}}{2016b}]%
        {Multi-dueling_bandits_and_their_application}
\bibfield{author}{\bibinfo{person}{Brian Brost}, \bibinfo{person}{Yevgeny
  Seldin}, \bibinfo{person}{Ingemar~J Cox}, {and} \bibinfo{person}{Christina
  Lioma}.} \bibinfo{year}{2016}\natexlab{b}.
\newblock \showarticletitle{Multi-dueling bandits and their application to
  online ranker evaluation}. In \bibinfo{booktitle}{\emph{Proceedings of the
  25th ACM International on Conference on Information and Knowledge
  Management}}. \bibinfo{pages}{2161--2166}.
\newblock


\bibitem[\protect\citeauthoryear{Chakrabarti, Kumar, Radlinski, and
  Upfal}{Chakrabarti et~al\mbox{.}}{2009}]%
        {chakrabarti2009mortal_online_advertisement}
\bibfield{author}{\bibinfo{person}{Deepayan Chakrabarti}, \bibinfo{person}{Ravi
  Kumar}, \bibinfo{person}{Filip Radlinski}, {and} \bibinfo{person}{Eli
  Upfal}.} \bibinfo{year}{2009}\natexlab{}.
\newblock \showarticletitle{Mortal multi-armed bandits}. In
  \bibinfo{booktitle}{\emph{Advances in Neural Information Processing
  Systems}}. \bibinfo{pages}{273--280}.
\newblock


\bibitem[\protect\citeauthoryear{Feller}{Feller}{[n. d.]}]%
        {feller1957_geometric_distribution}
\bibfield{author}{\bibinfo{person}{William Feller}.} \bibinfo{year}{[n.
  d.]}\natexlab{}.
\newblock \showarticletitle{An introduction to probability theory and its
  applications}.
\newblock \bibinfo{journal}{\emph{1957}} (\bibinfo{year}{[n. d.]}).
\newblock


\bibitem[\protect\citeauthoryear{Hofmann, Whiteson, and de~Rijke}{Hofmann
  et~al\mbox{.}}{2013}]%
        {hofmann2013_Information_Retrieval}
\bibfield{author}{\bibinfo{person}{Katja Hofmann}, \bibinfo{person}{Shimon
  Whiteson}, {and} \bibinfo{person}{Maarten de Rijke}.}
  \bibinfo{year}{2013}\natexlab{}.
\newblock \showarticletitle{Balancing exploration and exploitation in listwise
  and pairwise online learning to rank for information retrieval}.
\newblock \bibinfo{journal}{\emph{Information Retrieval}} \bibinfo{volume}{16},
  \bibinfo{number}{1} (\bibinfo{year}{2013}), \bibinfo{pages}{63--90}.
\newblock


\bibitem[\protect\citeauthoryear{Iizuka, Yoneda, and Seki}{Iizuka
  et~al\mbox{.}}{2019}]%
        {multileaving_2019}
\bibfield{author}{\bibinfo{person}{Kojiro Iizuka}, \bibinfo{person}{Takeshi
  Yoneda}, {and} \bibinfo{person}{Yoshifumi Seki}.}
  \bibinfo{year}{2019}\natexlab{}.
\newblock \showarticletitle{Greedy optimized multileaving for personalization}.
  In \bibinfo{booktitle}{\emph{Proceedings of the 13th ACM Conference on
  Recommender Systems}}. ACM, \bibinfo{pages}{413--417}.
\newblock


\bibitem[\protect\citeauthoryear{Kalyanakrishnan, Tewari, Auer, and
  Stone}{Kalyanakrishnan et~al\mbox{.}}{2012}]%
        {lucb}
\bibfield{author}{\bibinfo{person}{Shivaram Kalyanakrishnan},
  \bibinfo{person}{Ambuj Tewari}, \bibinfo{person}{Peter Auer}, {and}
  \bibinfo{person}{Peter Stone}.} \bibinfo{year}{2012}\natexlab{}.
\newblock \showarticletitle{PAC Subset Selection in Stochastic Multi-armed
  Bandits.}. In \bibinfo{booktitle}{\emph{Proceedings of the 29th International
  Conference on Machine Learning}}, Vol.~\bibinfo{volume}{12}.
  \bibinfo{pages}{655--662}.
\newblock


\bibitem[\protect\citeauthoryear{Kohli, Salek, and Stoddard}{Kohli
  et~al\mbox{.}}{2013}]%
        {kohli2013recommendation_systems}
\bibfield{author}{\bibinfo{person}{Pushmeet Kohli}, \bibinfo{person}{Mahyar
  Salek}, {and} \bibinfo{person}{Greg Stoddard}.}
  \bibinfo{year}{2013}\natexlab{}.
\newblock \showarticletitle{A fast bandit algorithm for recommendations to
  users with heterogeneous tastes}. In \bibinfo{booktitle}{\emph{Proceedings of
  the 27th AAAI Conference on Artificial Intelligence}}.
  \bibinfo{pages}{1135--1141}.
\newblock


\bibitem[\protect\citeauthoryear{Oosterhuis and de~Rijke}{Oosterhuis and
  de~Rijke}{2017}]%
        {multileaving_2017}
\bibfield{author}{\bibinfo{person}{Harrie Oosterhuis} {and}
  \bibinfo{person}{Maarten de Rijke}.} \bibinfo{year}{2017}\natexlab{}.
\newblock \showarticletitle{Sensitive and scalable online evaluation with
  theoretical guarantees}. In \bibinfo{booktitle}{\emph{Proceedings of the 26th
  ACM on Conference on Information and Knowledge Management}}. ACM,
  \bibinfo{pages}{77--86}.
\newblock


\bibitem[\protect\citeauthoryear{Page, Brin, Motwani, and Winograd}{Page
  et~al\mbox{.}}{1999}]%
        {PageRank}
\bibfield{author}{\bibinfo{person}{Lawrence Page}, \bibinfo{person}{Sergey
  Brin}, \bibinfo{person}{Rajeev Motwani}, {and} \bibinfo{person}{Terry
  Winograd}.} \bibinfo{year}{1999}\natexlab{}.
\newblock \bibinfo{booktitle}{\emph{The PageRank citation ranking: Bringing
  order to the web.}}
\newblock \bibinfo{type}{{T}echnical {R}eport}. \bibinfo{institution}{Stanford
  InfoLab}.
\newblock


\bibitem[\protect\citeauthoryear{Qin and Liu}{Qin and Liu}{2013}]%
        {Introducing_LETOR_Datasets}
\bibfield{author}{\bibinfo{person}{Tao Qin} {and} \bibinfo{person}{Tie{-}Yan
  Liu}.} \bibinfo{year}{2013}\natexlab{}.
\newblock \showarticletitle{Introducing {LETOR} 4.0 Datasets}.
\newblock \bibinfo{journal}{\emph{CoRR}}  \bibinfo{volume}{abs/1306.2597}
  (\bibinfo{year}{2013}).
\newblock
\urldef\tempurl%
\url{http://arxiv.org/abs/1306.2597}
\showURL{%
\tempurl}


\bibitem[\protect\citeauthoryear{Rahimi and Shakery}{Rahimi and
  Shakery}{2017}]%
        {interleav2017_sigir}
\bibfield{author}{\bibinfo{person}{Razieh Rahimi} {and} \bibinfo{person}{Azadeh
  Shakery}.} \bibinfo{year}{2017}\natexlab{}.
\newblock \showarticletitle{Online learning to rank for cross-language
  information retrieval}. In \bibinfo{booktitle}{\emph{Proceedings of the 40th
  International ACM SIGIR Conference on Research and Development in Information
  Retrieval}}. ACM, \bibinfo{pages}{1033--1036}.
\newblock


\bibitem[\protect\citeauthoryear{Robertson and Walker}{Robertson and
  Walker}{1994}]%
        {BM25}
\bibfield{author}{\bibinfo{person}{Stephen~E Robertson} {and}
  \bibinfo{person}{Steve Walker}.} \bibinfo{year}{1994}\natexlab{}.
\newblock \showarticletitle{Some simple effective approximations to the
  2-poisson model for probabilistic weighted retrieval}. In
  \bibinfo{booktitle}{\emph{SIGIR’94}}. Springer, \bibinfo{pages}{232--241}.
\newblock


\bibitem[\protect\citeauthoryear{Schuth, Hofmann, and Radlinski}{Schuth
  et~al\mbox{.}}{2015}]%
        {interleav2015_sigir}
\bibfield{author}{\bibinfo{person}{Anne Schuth}, \bibinfo{person}{Katja
  Hofmann}, {and} \bibinfo{person}{Filip Radlinski}.}
  \bibinfo{year}{2015}\natexlab{}.
\newblock \showarticletitle{Predicting search satisfaction metrics with
  interleaved comparisons}. In \bibinfo{booktitle}{\emph{Proceedings of the
  38th International ACM SIGIR Conference on Research and Development in
  Information Retrieval}}. ACM, \bibinfo{pages}{463--472}.
\newblock


\bibitem[\protect\citeauthoryear{Sui and Burdick}{Sui and Burdick}{2014}]%
        {sui2014clinical}
\bibfield{author}{\bibinfo{person}{Yanan Sui} {and} \bibinfo{person}{Joel
  Burdick}.} \bibinfo{year}{2014}\natexlab{}.
\newblock \showarticletitle{Clinical online recommendation with subgroup rank
  feedback}. In \bibinfo{booktitle}{\emph{Proceedings of the 8th ACM Conference
  on Recommender Systems}}. ACM, \bibinfo{pages}{289--292}.
\newblock


\bibitem[\protect\citeauthoryear{Sui, Zhuang, Burdick, and Yue}{Sui
  et~al\mbox{.}}{2017}]%
        {Multi-dueling_Bandits_with_Dependent_Arms}
\bibfield{author}{\bibinfo{person}{Yanan Sui}, \bibinfo{person}{Vincent
  Zhuang}, \bibinfo{person}{Joel~W. Burdick}, {and} \bibinfo{person}{Yisong
  Yue}.} \bibinfo{year}{2017}\natexlab{}.
\newblock \showarticletitle{Multi-dueling Bandits with Dependent Arms}. In
  \bibinfo{booktitle}{\emph{Proceedings UAI, Conference on Uncertainty in
  Artificial Intelligence}}.
\newblock


\bibitem[\protect\citeauthoryear{Thompson}{Thompson}{1933}]%
        {thompson1933}
\bibfield{author}{\bibinfo{person}{William~R Thompson}.}
  \bibinfo{year}{1933}\natexlab{}.
\newblock \showarticletitle{On the likelihood that one unknown probability
  exceeds another in view of the evidence of two samples}.
\newblock \bibinfo{journal}{\emph{Biometrika}} \bibinfo{volume}{25},
  \bibinfo{number}{3/4} (\bibinfo{year}{1933}), \bibinfo{pages}{285--294}.
\newblock


\bibitem[\protect\citeauthoryear{Urvoy, Clerot, F{\'e}raud, and Naamane}{Urvoy
  et~al\mbox{.}}{2013}]%
        {savage}
\bibfield{author}{\bibinfo{person}{Tanguy Urvoy}, \bibinfo{person}{Fabrice
  Clerot}, \bibinfo{person}{Raphael F{\'e}raud}, {and} \bibinfo{person}{Sami
  Naamane}.} \bibinfo{year}{2013}\natexlab{}.
\newblock \showarticletitle{Generic exploration and k-armed voting bandits}. In
  \bibinfo{booktitle}{\emph{Proceedings of the 30th International Conference on
  Machine Learning}}. \bibinfo{pages}{91--99}.
\newblock


\bibitem[\protect\citeauthoryear{Villar, Bowden, and Wason}{Villar
  et~al\mbox{.}}{2015}]%
        {villar2015clinical_trials}
\bibfield{author}{\bibinfo{person}{Sof{\'\i}a~S Villar}, \bibinfo{person}{Jack
  Bowden}, {and} \bibinfo{person}{James Wason}.}
  \bibinfo{year}{2015}\natexlab{}.
\newblock \showarticletitle{Multi-armed bandit models for the optimal design of
  clinical trials: benefits and challenges}.
\newblock \bibinfo{journal}{\emph{Statistical science: a review journal of the
  Institute of Mathematical Statistics}} \bibinfo{volume}{30},
  \bibinfo{number}{2} (\bibinfo{year}{2015}), \bibinfo{pages}{199}.
\newblock


\bibitem[\protect\citeauthoryear{Wang, Langley, Kim, McCord-Snook, and
  Wang}{Wang et~al\mbox{.}}{2018}]%
        {interleav2018_sigir}
\bibfield{author}{\bibinfo{person}{Huazheng Wang}, \bibinfo{person}{Ramsey
  Langley}, \bibinfo{person}{Sonwoo Kim}, \bibinfo{person}{Eric McCord-Snook},
  {and} \bibinfo{person}{Hongning Wang}.} \bibinfo{year}{2018}\natexlab{}.
\newblock \showarticletitle{Efficient exploration of gradient space for online
  learning to rank}. In \bibinfo{booktitle}{\emph{The 41st International ACM
  SIGIR Conference on Research \& Development in Information Retrieval}}. ACM,
  \bibinfo{pages}{145--154}.
\newblock


\bibitem[\protect\citeauthoryear{Yue, Broder, Kleinberg, and Joachims}{Yue
  et~al\mbox{.}}{2012}]%
        {IF_JCSS2012}
\bibfield{author}{\bibinfo{person}{Yisong Yue}, \bibinfo{person}{Josef Broder},
  \bibinfo{person}{Robert Kleinberg}, {and} \bibinfo{person}{Thorsten
  Joachims}.} \bibinfo{year}{2012}\natexlab{}.
\newblock \showarticletitle{The k-armed dueling bandits problem}.
\newblock \bibinfo{journal}{\emph{J. Comput. System Sci.}}
  \bibinfo{volume}{78}, \bibinfo{number}{5} (\bibinfo{year}{2012}),
  \bibinfo{pages}{1538--1556}.
\newblock


\bibitem[\protect\citeauthoryear{Yue and Joachims}{Yue and Joachims}{2011}]%
        {btm}
\bibfield{author}{\bibinfo{person}{Yisong Yue} {and} \bibinfo{person}{Thorsten
  Joachims}.} \bibinfo{year}{2011}\natexlab{}.
\newblock \showarticletitle{Beat the mean bandit}. In
  \bibinfo{booktitle}{\emph{Proceedings of the 28th International Conference on
  Machine Learning}}. \bibinfo{pages}{241--248}.
\newblock


\bibitem[\protect\citeauthoryear{Zoghi, Karnin, Whiteson, and De~Rijke}{Zoghi
  et~al\mbox{.}}{2015}]%
        {Copeland_dueling_bandits}
\bibfield{author}{\bibinfo{person}{Masrour Zoghi}, \bibinfo{person}{Zohar~S
  Karnin}, \bibinfo{person}{Shimon Whiteson}, {and} \bibinfo{person}{Maarten
  De~Rijke}.} \bibinfo{year}{2015}\natexlab{}.
\newblock \showarticletitle{Copeland dueling bandits}. In
  \bibinfo{booktitle}{\emph{Advances in Neural Information Processing
  Systems}}. \bibinfo{pages}{307--315}.
\newblock


\bibitem[\protect\citeauthoryear{Zoghi, Whiteson, Munos, and De~Rijke}{Zoghi
  et~al\mbox{.}}{2014}]%
        {Relative_upper_confidence_bound}
\bibfield{author}{\bibinfo{person}{Masrour Zoghi}, \bibinfo{person}{Shimon
  Whiteson}, \bibinfo{person}{Remi Munos}, {and} \bibinfo{person}{Maarten
  De~Rijke}.} \bibinfo{year}{2014}\natexlab{}.
\newblock \showarticletitle{Relative upper confidence bound for the K-armed
  dueling bandit problem}. In \bibinfo{booktitle}{\emph{Proceedings of the 31st
  International Conference on Machine Learning}}. \bibinfo{pages}{II--10}.
\newblock


\end{thebibliography}

\clearpage
\section*{A \quad  Proof of Theorem \ref{theorem1_DoublerBAI}}

\thmdoubler*

\begin{proof}
	For convenience, we recall the notation from the proof sketch of \cref{theorem1_DoublerBAI}.  $B(\delta)$ denotes the supremum of the expected regret of the BAIM $S$ to identify the best arm with probability at least $1-\delta$. We also denote the sample complexity of the BAIM $S$ to identify the best arm with probability at least $1-\delta$ by $\widetilde{B}(\delta)$. Because each sample on arm $x_i \in \mathcal{X}$ incurs regret $\Delta_{i} \leq1 $, $B(\delta)$ and $\widetilde{B}(\delta)$ have the same order of $O(H \ln(\frac{H}{\delta}))$ \cite{lucb}. Thus, we can assume  $B(\delta) = c_1 H \ln(\frac{H}{\delta})$ and $\widetilde{B}(\delta)= c_2 H \ln(\frac{H}{\delta})$, where  $c_2 \geq c_1 \geq 1$. Let $R_t^{\textup{left}}$ and $R_t^{\textup{right}}$ denote the regret incurred by the left and right arm of the played two arms $(x_t, y_t)$ at time-step $t$, respectively .
	
	We firstly consider $R_t^{\textup{right}}$ of the exploration stage. 
	From the definition of regret in the dueling bandits problem, we have that
	\begin{align}
	&\mathbb{E}[\sum \limits_{t=1} \limits^{\tau_i^{\textup{explore}}} R_t^{\textup{right}}]\nonumber \\
	= & \mathbb{E}[\sum \limits_{t=1} \limits^{\tau_i^{\textup{explore}}}\frac{\mu(x_1)-\mu(y_t)}{2}]   \nonumber
	\\
	\label{eq_1122}= & \mathbb{E}[\sum \limits_{t=1} \limits^{\tau_i^{\textup{explore}}}( \frac{\mu(x_1)-\mu(\bar{x}_i)+1}{2} - \frac{\mu(y_t)-\mu(\bar{x}_i)+1}{2} )] 
	\\
	\label{eq_1123}\leq & B(\frac{1}{\tau_{i+1}}).  
	\end{align}
	The last inequality holds because  \cref{eq_1122} is the regret of the  BAI game in epoch $i$, and \cref{eq_1123} is a upper bound of that regret.
	
	In the stage of exploitation, the right arm is simply chosen as the identified  best arm of the exploration stage, with error probability at most $\frac{1}{\tau_{i+1}}$. 
	
	Thus, we can bound $\mathbb{E}[\sum \limits_{t=1} \limits^{\tau_i } R_t^{\textup{right}}]$ by
	\begin{align}
	\mathbb{E}[\sum \limits_{t=1} \limits^{\tau_i } R_t^{\textup{right}}]& \leq  (1-\frac{1}{\tau_{i+1}})  B(\frac{1}{\tau_{i+1}}) + \frac{1}{\tau_{i+1}} O(\tau_i)  \nonumber
	\\
	&\leq  B(\frac{1}{\tau_{i+1}}) + O(1). \nonumber
	\end{align} 
	
	Now we consider $R_t^{\textup{left}}$. In epoch $i$, the left arm $\bar{x}_i$ is chosen in two cases: (i) $\bar{x}_i$ is chosen as the identified best arm $\hat{x}_{i-1}$ of epoch $i-1$ when $S$ terminates and returns $\hat{x}_{i-1}$ in epoch $i-1$, and (ii) $\bar{x}_i$ is randomly chosen from $\mathcal{X}$ otherwise.

	In case (i), the probability of $\hat{x}_{i-1} \neq x_1$ is at most $\frac{1}{\tau_{i}}$, and thus we have 
	\begin{align}
	\mathbb{E}[\sum \limits_{t=1} \limits^{\tau_i } R_t^{\textup{left}}] \leq \frac{1}{\tau_{i}} \cdot O(\tau_i) + (1-\frac{1}{\tau_{i}}) \cdot 0 = O(1). \nonumber
	\end{align}

	In case (ii), we simply have
	\begin{align}
	\mathbb{E}[\sum \limits_{t=1} \limits^{\tau_i } R_t^{\textup{left}}] = O(\tau_i). \nonumber
	\end{align}
	
	Thus, in case (i), the regret is upper bounded by $B(\frac{1}{\tau_{i+1}}) + O(1)$. In case (ii), the regret is upper bounded by $O(\tau_i)$. 

	Below we bound the   regret  in case (i) and (ii), respectively.

	First consider case (i), where we bound the  regret using similar techniques of doubling trick \cite{doubling_trick_auer2010ucb,What_Doubling_Tricks}. Let $L_T$ denote the  number of epochs up to time $T$. According to the definition of $\{T_i\}_{i \in \mathbb{N}}$  (\emph{i.e.}, $T_i=  \lfloor  a^{b^i} \rfloor,\  a,b>1$), we have that $\forall L_T>1$,
	\begin{align}
	L_T -1 &=\left \lceil \ln_b \ln_a T \right \rceil \leq  \ln_b \ln_a T+1. \label{bound_L_T}
	\end{align}
	Asymptotically for $i$ and $T \rightarrow \infty$, $T_i=O(a^{b^i})$ and $L_T = O(\ln \ln T)$.
	
	The  expected regret up to time $T$ in case (a)  can be bounded by (the index of epoch $i$ starts from $0$)
	\begin{align}
	&\sum \limits_{i=0} \limits^{L_T-1 } \big(B(\frac{1}{\tau_{i+1}}) + O(1) \big) \nonumber
	\\
	\leq& \sum \limits_{i=0} \limits^{L_T } (c_1 H \ln(H \tau_{i})+ O(1)) \nonumber
	\\
	\leq &\sum \limits_{i=0} \limits^{L_T } (c_1 H \ln(H a^{b^i})+ O(1))  \nonumber
	\\
	= &\sum \limits_{i=0} \limits^{L_T } (c_1 H \ln H + c_1 H  \ln (a) b^i + O(1))  \nonumber
	\\
	= &O(H \ln H \ln \ln T) +O(\ln \ln T) \nonumber \\
	&+ c_1 H \ln a \sum \limits_{i=0} \limits^{L_T } b^i,  \label{case_a}
	\end{align}
	where $\sum \limits_{i=0} \limits^{L_T } b^i$ can be bounded by
	$
	\sum \limits_{i=0} \limits^{L_T } b^i \leq \frac{b}{b-1}b^{L_T} 
	\leq \frac{b^3}{b-1} \ln_a T 
	=\frac{b^3}{b-1} \frac{\ln T}{\ln a} 
	$ (using \cref{bound_L_T}). 
	
	Thus, we can bound \cref{case_a} by 
	\begin{align}
	\cref{case_a}
	\leq &O(H \ln H \ln \ln T)+O(\ln \ln T) \nonumber \\
	&+c_1 H \frac{b^3}{b-1} \ln T \nonumber
	\\
	= &O(H \ln H \ln \ln T)+O(\ln \ln T) \nonumber \\
	&+O(H \ln T).  \label{case_a_result}
	\end{align}
	
	Then, we consider case (ii), where we will prove the regret is independent of $T$. In epoch $i$, with probability at least $1-\frac{1}{\tau_{i+1}}$, S will terminate and return the identified best arm $\hat{x}_{i}$ after $\widetilde{B}(\frac{1}{\tau_{i+1}})= c_2 H \ln(H \tau_{i+1})$ time-steps.  For ease of analysis, we regard the scenario where  S does not terminate after $ c_2 H \ln(H \tau_{i+1})$ time-steps as the scenario where S terminates but returns a wrong arm and our analysis still holds. 
	
	Note that if 
	\begin{equation}\label{case_b}c_2 H \ln (H \tau_{i+1})  \leq \tau_i\end{equation}
	holds, S will terminate and return $\hat{x}_{i}$, and thus case (ii) in epoch $i+1$ will not occur. For large enough $\tau_i$, $c_2 H \ln (H \tau_{i+1})  \leq \tau_i$ must holds. This is because that the left side grows logarithmically  with  $\tau_i$ and the right side grows linearly ($\tau_i$ is defined to grow exponentially). Thus, case (ii)  only occurs in early epoch, for which the previous epoch does not satisfy this condition. 
	
	For $i>0$,  \cref{case_b} can be written as
	\begin{align}
	c_2 H \ln H + c_2 H \ln(T_{i+1}-T_i)  \leq T_i - T_{i-1}. \nonumber
	\end{align}

	Notice that $T_{i+1}-T_i \leq T_{i+1} \leq (T_i+1)^b$, and similarly, $T_{i-1} \leq (T_i+1)^{\frac{1}{b}}$. Thus, $\forall T_i \geq T_0:=\left \lceil (\frac{1}{2})^{\frac{b}{1-b}} -1\right \rceil$, $T_i - T_{i-1} \geq T_i -  (T_i+1)^{\frac{1}{b}} \geq \frac{1}{2} T_i -\frac{1}{2}$.

	Then we know that for any epoch $i$ such that $ T_i \geq \widetilde{T}= \max \{ T_0,$ $\left \lceil  20 c_2 b H \ln (c_2 b H) -1\right \rceil \}$,  S will terminate  and return the identified best arm $\hat{x}_i$.   
	
	This is because that 
	\begin{align*} 
	c_2 H \ln H + c_2 H b \ln(T_i+1) \overset{a}{\leq} \frac{1}{2} T_i - \frac{1}{2} \overset{a}{\leq} T_i - T_{i-1}
	\end{align*}
	where (a) uses $T_i \geq \left \lceil  20 c_2 b H \ln (c_2 b H) -1\right \rceil $ and (b) uses $T_i \geq T_0$.
	
	Now we can obtain $L-1 \leq \left \lceil \ln_b \ln_a \widetilde{T} \right \rceil \leq  \ln_b \ln_a \widetilde{T}+1 $, and thus $T_{L-1} \leq a^{b^{\ln_b \ln_a \widetilde{T}+1}}=\widetilde{T}^b$, where $L$ is the number of the  epochs in which case (ii) occurs, \emph{i.e.}, case (ii) occurs in epoch $0,1,...,L-1$ and \cref{case_b} firstly holds in epoch $L-1$.


	
	
	Therefore, we can bound the expected regret in case (ii)  by 
	\begin{align}
	\mathbb{E}[\sum \limits_{i=0} \limits^{L-1 } R^i]  = & \sum \limits_{i=0} \limits^{L-1} O(\tau_i) \nonumber 
	\\
	= & O(T_{L-1})  \nonumber 
	\\
	\leq & O(\widetilde{T}^b)  \nonumber 
	\\
	= & O\big((H \ln H)^b \big). \label{case_b_result}
	\end{align}
	
	Summing up  the expected regret in case (i) (\cref{case_a_result}) and (ii) (\cref{case_b_result}), we obtain the result of \cref{theorem1_DoublerBAI}.
	
\end{proof}

\section*{B \quad  Proof of Theorem \ref{theorem2_MultiSBM-Feedback}}

In order to prove \cref{theorem2_MultiSBM-Feedback}, we firstly introduce the following lemma.

\begin{lemma}
	\label{lemma1_SBM_robustness}
	When running the implement of the SBM (Algorithm \ref{algorithm_UCB_Feedback}) with $\alpha>0$, the number of times a suboptimal arm $x_i$ has been pulled up to time $T$, which is denoted by $\rho_i(T)$, satisfies
	\begin{align}
	\forall s \geq \frac{4(\alpha+4) }{\Delta_i^2}\ln T, \  \Pr[\rho_i(T) \geq s] < \frac{4}{\alpha} \Big(\frac{s}{2} \Big)^{-\alpha}. \nonumber
	\end{align}
\end{lemma}

\begin{proof}
	
	Our analysis follows similar techniques as that in \cite{Reducing_dueling_bandits}. For ease of notation, we define $\beta:=\alpha+2$ and  $u_i(t):=\frac{2\beta \ln t}{\Delta_i^2}$. Recall the notation defined in Algorithm \ref{algorithm_UCB_Feedback}. $\rho_i(t)$ denotes the number of times  arm $x_i \in \mathcal{X}$ has been pulled up to time $t$ and $s_i(t)$ denotes the number of times arm $x_i$ has received the additional feedback up to time $t$.
	
	At time-step $t$, if a suboptimal arm $x_i$ was chosen, one of the following three events must be true. 
	\begin{align} 
	\mathcal{E}_t &:=\{ \rho_i(t)+s_i(t) < u_i(t) \}\nonumber
	\\
	\mathcal{F}_t &:= \{\hat{\mu}_i \geq \mu_i + \sqrt{ \frac{\beta \ln t}{2(\rho_i(t)+s_i(t)) } }\} \nonumber
	\\
	\mathcal{G}_t &:=\{ \hat{\mu}_1 + \sqrt{ \frac{\beta \ln t}{2(\rho_1(t)+s_1(t)) } } \leq \mu_1 \} \nonumber
	\end{align}
	
	If all three are false, we have
	\begin{align}
	& \hat{\mu}_1 + \sqrt{ \frac{\beta \ln t}{2(\rho_i(t)+s_i(t)) } } \nonumber \\
	>&\mu_1 \nonumber \\
	=&\mu_i + \Delta_i \nonumber
	\\
	\geq & \mu_i + 2 \sqrt{ \frac{\beta \ln t}{2(\rho_i(t)+s_i(t)) } } \nonumber 
	\\
	>& \hat{\mu}_i + \sqrt{ \frac{\beta \ln t}{2(\rho_i(t)+s_i(t)) } }, \nonumber
	\end{align}
	and then arm $i$ cannot be chosen.
	
	When $\rho_i(t) \geq u_i(T)$, event $\mathcal{E}_t$ is false. Thus, we have 
	\begin{align}
	\mathbb{E}[\rho_i(T) - u_i(T)] \leq \sum \limits_{t=u_i(T)+1} \limits^{T} \Pr[\mathcal{F}_t \lor \mathcal{G}_t]. \nonumber
	\end{align}
	Using the Chernoff-Hoeffding bound, we can bound the probability of event $\mathcal{F}_t$ occurring by
	\begin{align}
	& \Pr[\mathcal{F}_t]  \leq \Pr \Bigg[\  \exists (\rho_i(t)+s_i(t)) \in [2t]: \nonumber \\
	& \hat{\mu}_i \geq \mu_i + \sqrt{ \frac{\beta \ln t}{2(\rho_i(t)+s_i(t)) } } \  \Bigg] \leq 2t \cdot  t^{-\beta}=2 t^{1-\beta}. \nonumber
	\end{align}
	Analogously, $\Pr[\mathcal{G}_t] \leq 2 t^{1-\beta}$.  Thus, we have
	\begin{align}
	\mathbb{E}[\rho_i(T) - u_i(T)] \leq & \sum \limits_{t=u_i(T)+1} \limits^{T} 2 \cdot 2 t^{1-\beta} \nonumber
	\\
	\leq & \frac{4}{\beta -2} \Big(\frac{2\beta\ln T}{\Delta_i^2}\Big)^{2-\beta}.  \nonumber
	\end{align}
	
	Let $\rho^s_i(T)$ denote the number of times arm $x_i$ has been pulled between time $s$ and $T$. For $s \geq \frac{2\beta \ln T}{\Delta_i^2}$, we have
	\begin{align}
	\mathbb{E}[\rho^s_i(T) - u_i(T)] \leq \sum \limits_{t=s} \limits^{T} 4 t^{1-\beta} \leq \frac{4}{\beta -2} s^{2-\beta}.  \nonumber
	\end{align}
	
	Assuming that arm $x_i$ has been pulled at least $s \geq \frac{4(\beta+2) \ln T}{\Delta_i^2}$ times up to time $T$, we have  $\rho^{s- u_i(T)-1}_i(T) \geq u_i(T)+1$. Thus, we can bound $\Pr[\rho_i(T) > s]$ by 
	\begin{align}
	\Pr[\rho_i(T) \geq s] \leq & \Pr[ \rho^{s- u_i(T)-1}_i(T) -u_i(T) \geq 1 ] \nonumber
	\\
	\overset{(a)}{\leq} & \mathbb{E}[ \rho^{s- u_i(T)-1}_i(T) -u_i(T)]  \nonumber
	\\
	\leq &  \frac{4}{\beta -2} (s- u_i(T)-1)^{2-\beta}  \nonumber
	\\
	\overset{(b)}{\leq} & \frac{4}{\beta -2} \Big( \frac{s}{2} \Big)^{2-\beta},  \nonumber
	\end{align}
	where (a) uses Markov's inequality and (b) uses $s \geq \frac{4(\beta+2) \ln T}{\Delta_i^2} \geq 2 u_i(T)+2$.

	Since $\beta:=\alpha+2$, we  obtain the result of \cref{lemma1_SBM_robustness}.
	
\end{proof}

From \cref{lemma1_SBM_robustness}, we know that our SBM's implement (Algorithm \ref{algorithm_UCB_Feedback}) also satisfies the $\alpha$-robustness  defined in \cite{Reducing_dueling_bandits}, with the constant factor of the probability slightly enlarged. 

Therefore, the analysis in (Theorem 4.2 in \cite{Reducing_dueling_bandits}) still holds. Let $\tau_{xy}(\tau_x(T))$ denote the  number of times SBM $S_x$ ($x \in \mathcal{X} \setminus \{x_1\}$) has advanced suboptimal arm $y$ up to time $T$. We have that with parameter $\alpha=\max\{3,\frac{\ln K}{\ln \ln T} \}$, $\mathbb{E}[\tau_{xy}(\tau_x(T))]$ is bounded by 
\begin{align}
\mathbb{E}[\tau_{xy}(\tau_x(T))] =  O \Big (\frac{\alpha} {\Delta_y^{2}} \big(\ln\ln T  + & \ln K \nonumber
\\
&+ \ln ( \frac{1}{\Delta_x}) \big) \Big). \label{eq_theorem2_loglog_term} 
\end{align}
This conclusion will be used in the following proof of \cref{theorem2_MultiSBM-Feedback}.

\thmMultiSBMfeedback*

\begin{proof} Let $\tau_x(t)$ denote the number of times SBM $S_x \  (x \in \mathcal{X})$ has been queried up to time $t$.  Let $R_x(t')$ denote the regret seen by $S_x$ up to its internal time $t'$.
	Let  $R_{xy}(\tau_x(t))$ denote the regret due to $S_x$ advancing suboptimal arm $y$ up to time $t$.
	
	In MultiSBM-Feedback, the right arm in each time-step equals to the left arm in the next time-step. Thus, we have that the  expected regret $\mathbb{E}[R_T]$ up to time $T$ can be bounded by
	\begin{align}
	\mathbb{E}[R_T]  \leq  0.5+ \mathbb{E}[R_{x_1}(T)]+ \mathbb{E}[\sum \limits_{y\neq x_1} \sum \limits_{x\neq x_1} R_{xy}(\tau_x(T)) ]. \nonumber
	\end{align}
	The analysis of $\mathbb{E}[\sum \limits_{y\neq x_1} \sum \limits_{x\neq x_1} R_{xy}(\tau_x(T)) ]$ follows the similar line of \cite{Reducing_dueling_bandits}.
	In the following we focus on  $\mathbb{E}[R_{x_1}(T)]$.
	
	We inherit the notation and reasoning in the proof of \cref{lemma1_SBM_robustness}.
	In MultiSBM-Feedback, up to any time $t$, the number of times arm $x_1$ being the left arm equals to the number of times arm $x_1$ being the right arm.
	Subtracting the number of times $(x_1, x_1)$ being played for both side, we have that the number of times $S_{x_1}$ advancing suboptimal arms equals to the number of times $S_{x_1}$ receives additional feedback from suboptimal arms. Thus, in SBM $S_{x_1}$,
	\begin{align}
	\forall t, \  \sum \limits_{i>1} \rho_i(t) =\sum   \limits_{i>1} s_i(t).   \label{eq_rho_equal_s}
	\end{align}

	At time-step $t$, we denote the arm being pulled by $I_t$ and the arm being observed through GetAdditionalFeedback by $A_t$. 
	We also define $\{\Pi(x)\}$ to be the indicator function of the event $\Pi(x)$ for any predicate $\Pi(x)$.
	In SBM $S_{x_1}$, up to its internal time $T$, for any suboptimal arm $x_i$, we have
	
	\begin{align}
	& \hspace*{3em} \rho_i(T)+s_i(T) \nonumber
	\\
	= & \sum \limits_{t=1} \limits^{T}\{ I_t=x_i \} + \sum \limits_{t=1} \limits^{T}\{ A_t=x_i \} \nonumber
	\\
	\leq& u_i(T)+ \sum \limits_{t=t_0+1} \limits^{T}\{ I_t=x_i, \rho_i(t)+s_i(t) \geq  u_i(T) \} \nonumber
	\\
	&+ \sum \limits_{t=t_0+1} \limits^{T}\{ A_t=x_i, \rho_i(t)+s_i(t) \geq  u_i(T) \} \nonumber
	\\
	\leq& u_i(T)+ \sum \limits_{t=t_0+1} \limits^{T}\{ I_t=x_i, \rho_i(t)+s_i(t) \geq  u_i(t) \} \nonumber
	\\
	&+ \sum \limits_{t=t_0+1} \limits^{T}\{ A_t=x_i, \rho_i(t)+s_i(t) \geq  u_i(t) \} \nonumber
	\\
	\leq& u_i(T)+ \sum \limits_{t=1} \limits^{\infty}\{ \mathcal{F}_t \lor \mathcal{G}_t \} + 1 \nonumber
	\\
	&+ \underbrace{ \sum \limits_{t=t_0+2} \limits^{T}\{ A_t=x_i, \rho_i(t-1)+s_i(t-1) \geq  u_i(t-1) \} }_{\Gamma}. \label{term_Gamma}
	\end{align}
	where $t_0 \  (1 \leq t_0 \leq T)$ denotes the time when $\rho_i(t_0)+s_i(t_0)=u_i(T)$ holds. If such $t_0$ does not exist, the inequality still holds.

	Event $A_t=x_i$ occurring implies that in MultiSBM-Feedback, the played pair of arms is $(x_i, x_1)$. This occurs only if $x_i$ was the right one in the previously played pair of arms, \emph{i.e.}, $(z_t, x_i)$. In other words, some SBM $S_{z_t}$ has advanced suboptimal arm $x_i$ before event $A_t=x_i$ occurs. However, at time-step $t=t_0+1,...,T$, $S_{x_1}$ will not advance $x_i$ unless either of event $\mathcal{F}_t$ or $\mathcal{G}_t$ occurs. Thus, we can bound term $\Gamma$ by

	\begin{align}
	\Gamma	=& \sum \limits_{t=t_0+2} \limits^{T}\{ A_t=x_i, \rho_i(t-1)+s_i(t-1) \geq  u_i(t-1),  \nonumber 
	\\
	& \hspace*{18em} z_t=x_1\}  \nonumber 
	\\
	& +  \sum \limits_{t=t_0+2} \limits^{T}\{ A_t=x_i, \rho_i(t-1)+s_i(t-1) \geq  u_i(t-1),   \nonumber 
	\\
	&  \hspace*{18em}  z_t \neq x_1\}  \nonumber 
	\\
	\leq & \sum \limits_{t=t_0+2} \limits^{T}\{ A_t=x_i, \rho_i(t-1)+s_i(t-1) \geq  u_i(t-1),   \nonumber 
	\\
	& \hspace*{6em} I_{t-1}=x_i  \}  +  \sum \limits_{z \neq x_1} \sum \limits_{t=1} \limits^{\widetilde{T}} \{S_z \text{ advance } x_i\} \nonumber
	\\
	\leq& \sum \limits_{t=1} \limits^{\infty}\{ \mathcal{F}_t \lor \mathcal{G}_t \} +  \sum \limits_{z \neq x_1} \tau_{z x_i}(\tau_{z}(\widetilde{T})), \nonumber
	\end{align}
	where $\widetilde{T}$ denotes the external time in MultiSBM-Feedback when the internal time in SBM $S_{x_1}$ is $T$, \emph{i.e.}, $\tau_{x_1}(\widetilde{T})=T$.

	Thus, \cref{term_Gamma} can be bounded by
	\begin{align}
	& \rho_i(t)+s_i(t) \nonumber
	\\
	\leq & u_i(T)+ 2 \sum \limits_{t=1} \limits^{\infty}\{ \mathcal{F}_t \lor \mathcal{G}_t \} + 1  
	+ \sum \limits_{z \neq x_1} \tau_{z x_i}(\tau_{z}(\widetilde{T})).  \nonumber
	\end{align}
	
	Taking summation over $i>1$ and using \cref{eq_rho_equal_s}, we have
	\begin{align}
	& \mathbb{E}[2\sum \limits_{i>1}\rho_i(T)] \nonumber
	\\
	= & \mathbb{E}[\sum \limits_{i>1}\rho_i(T)+\sum \limits_{i>1}s_i(T)] \nonumber
	\\
	\leq & \sum \limits_{i>1}u_i(T) + \frac{8}{\beta -2}K + K + \sum \limits_{x_i \neq x_1} \sum \limits_{z \neq x_1}\mathbb{E}[\tau_{z x_i}(\tau_{z}(\widetilde{T}))] \nonumber
	\\
	\leq & \sum \limits_{i>1}\frac{2\beta}{\Delta_i^2}\ln T+ \frac{\beta +6}{\beta -2}K + \sum \limits_{x_i \neq x_1} \sum \limits_{z \neq x_1}\mathbb{E}[\tau_{z x_i}(\tau_{z}(\widetilde{T}))]. \nonumber
	\end{align}
	
	Since $\beta:=\alpha+2$, we have
	\begin{align}
	\mathbb{E}[\sum \limits_{i>1}\rho_i(T)]
	\leq & \sum \limits_{i>1}\frac{(\alpha+2)}{\Delta_i^2}\ln T+ \frac{\alpha+8}{2\alpha}K \nonumber
	\\
	& + \frac{1}{2} \sum \limits_{x_i \neq x_1} \sum \limits_{z \neq x_1}\mathbb{E}[\tau_{z x_i}(\tau_{z}(\widetilde{T}))]. \nonumber
	\end{align}
	
	Therefore,  we can bound $\mathbb{E}[R_{x_1}(T)] $  by
	\begin{align}
	\mathbb{E}[R_{x_1}(T)] \leq & \mathbb{E}[\sum \limits_{i>1}\rho_i(T)]] \cdot \Delta_{max} \nonumber
	\\
	\leq & \sum \limits_{i>1}\frac{(\alpha+2)\Delta_{max} }{\Delta_i^2}\ln T+ \frac{(\alpha+8)\Delta_{max}}{2\alpha}K \nonumber \\
	& + \frac{\Delta_{max}}{2}  \sum \limits_{y\neq x_1} \sum \limits_{x\neq x_1} \mathbb{E}[\tau_{xy}(\tau_{x}(\widetilde{T}))].  \nonumber
	\end{align}
	
	Using \cref{eq_theorem2_loglog_term}, the expected regret of MultiSBM-Feedback up to time $\widetilde{T}$ is bounded by
	
	\begin{align}
	& \mathbb{E}[R(\widetilde{T})] \nonumber
	\\
	\overset{(a)}{\leq} & 0.5+ \mathbb{E}[R_{x_1}(T)]+ \mathbb{E}[\sum \limits_{y\neq x_1} \sum \limits_{x\neq x_1} R_{xy}(\tau_x(\widetilde{T})) ]  \nonumber
	\\
	\leq & 0.5+   \sum \limits_{i>1}\frac{(\alpha+2)\Delta_{max} }{\Delta_i^2}\ln T+ \frac{(\alpha+8)\Delta_{max}}{2\alpha}K  \nonumber 
	\\
	& + \frac{\Delta_{max}}{2}  \sum \limits_{y\neq x_1} \sum \limits_{x\neq x_1} \mathbb{E}[\tau_{xy}(\tau_{x}(\widetilde{T}))] \nonumber
	\\
	&+  \sum \limits_{y\neq x_1} \sum \limits_{x\neq x_1} \mathbb{E}[R_{xy}(\tau_{x}(\widetilde{T}))]  \nonumber
	\\
	\leq &  \sum \limits_{i>1}\frac{(\alpha+2)\Delta_{max}}{\Delta_i^2} \ln \widetilde{T} + \frac{(\alpha+8)\Delta_{max}}{2\alpha}K   \nonumber
	\\
	& + \sum \limits_{j>1} \sum \limits_{i>1} O\Big( \frac{\alpha \Delta_{max}}{\Delta_j^2} \big( \ln\ln \widetilde{T} + \ln K + \ln(\frac{1}{\Delta_i}) \big) \Big), \nonumber
	\end{align}
	where (a) holds since  $\tau_{x_1}(\widetilde{T})=T$.

	Replacing $\widetilde{T}$ with $T$, we obtain
	\begin{align}
		\mathbb{E}[R(T)]
		\leq &  \sum \limits_{i>1}\frac{(\alpha+2)\Delta_{max}}{\Delta_i^2} \ln T + \frac{(\alpha+8)\Delta_{max}}{2\alpha}K   \nonumber
		\\
		& + \sum \limits_{j>1} \sum \limits_{i>1} O\Big( \frac{\alpha \Delta_{max}}{\Delta_j^2} \big( \ln\ln T + \ln K + \ln(\frac{1}{\Delta_i}) \big) \Big). \label{eq_new_result}
	\end{align}
	
	Note that our analysis uses a novel technique to bound $ \rho_i(t)+s_i(t)$. In addition, the standard analysis procedure that bounds $\rho_i(t)$ by $u_i(t)$ in MultiSBM \cite{Reducing_dueling_bandits} still holds in our analysis. Thus, $\mathbb{E}[R_{x_1}(T)]$ can be also bounded by
	$$
	\mathbb{E}[R_{x_1}(T)] \leq \sum \limits_{i > 1} \frac{2(\alpha+2)}{\Delta_i} \ln T + \frac{4 \Delta_{max}}{\alpha} K.
	$$
	Then, we can obtain
	\begin{align}
		\mathbb{E}[R(T)] 
		\leq & 0.5+ \mathbb{E}[R_{x_1}(T)]+ \mathbb{E}[\sum \limits_{y\neq x_1} \sum \limits_{x\neq x_1} R_{xy}(T) ]  \nonumber
		\\
		\leq & \sum \limits_{i > 1} \frac{2(\alpha+2)}{\Delta_i} \ln T + \frac{4 \Delta_{max}}{\alpha} K \nonumber
		\\
		& + \sum \limits_{j>1} \sum \limits_{i>1} O \Big (\frac{\alpha} {\Delta_y} \big(\ln\ln T  +  \ln K 
		+ \ln ( \frac{1}{\Delta_x}) \big) \Big). \label{eq_originalResult}
	\end{align}
	\cref{theorem2_MultiSBM-Feedback} follows from  \cref{eq_new_result,eq_originalResult}.
	
\end{proof}

\section*{C \quad  Proof of Theorem \ref{theorem3_MultiRUCB}}
In order to prove \cref{theorem3_MultiRUCB},   we quote a lemma (\cref{lemma2_CH_bound}) in \cite{Relative_upper_confidence_bound}  and introduce another five lemmas (\cref{lemma3_case_a,lemma4_case_b_case_c-1,lemma5_case_c-2,lemma6_former_term,lemma7_later_term}) as follows.

\begin{lemma}
	\label{lemma2_CH_bound}
	Let $\mathbf{P}:=[p_{ij}]$ be the preference matrix of a K-armed dueling bandit problem with arms $\{x_1,...,x_K\}$. Then, for any dueling bandit algorithm and any $\alpha>\frac{1}{2}$ and $\delta>0$, we have
	\begin{align}
	P\Big (\forall t> C(\delta), i,j, p_{ij} \in [l_{ij}(t), u_{ij}(t)]\Big ) > 1-\delta. \nonumber
	\end{align}
\end{lemma} 

This lemma is quoted from (Lemma 1 in \cite{Relative_upper_confidence_bound}).

\begin{lemma}
	\label{lemma3_case_a}
	With probability  at least $1-\delta$, $\forall t> C(\delta)$, \\ $x_1 \in \mathcal{C}$.
\end{lemma}

\begin{proof}
	Using \cref{lemma2_CH_bound}, we have that with probability  at least $1-\delta$, $\forall t> C(\delta),i, \   u_{1i}(t) \geq p_{1i}(t) \geq \frac{1}{2}$. This concludes the proof of \cref{lemma3_case_a}.
\end{proof}

For ease of notation, we recall the notation defined in the proof sketch of \cref{theorem3_MultiRUCB}. Case (a), (b) and (c) denote the  three mutually exclusive cases corresponding to Line 11, Line 14 and Line 16 in  Algorithm \ref{algorithm_MultiRUCB}, respectively. Case (c-1) and (c-2) respectively denote the two mutually exclusive  situations  in case (c), \emph{i.e.}, $x_1 \in \mathcal{A}_t$ and $x_1 \notin \mathcal{A}_t$. 

We also introduce the following notation. Let $ N(t)$ denote the number of times  $\mathcal{A}_t \neq \{x_1\}$ up to time $t$. Let $\widetilde{N}_{ij}(t)$  ($\widetilde{N}_{ij}(t)$=$\widetilde{N}_{ji}(t)$) denote the number of observed dueling outcomes  of $x_i$ and $x_j$ between time $C(\delta)+1$ and $t$. 
Let $\widetilde{N}^{b}(t)$, $\widetilde{N}^{c}_1(t)$ and $\widetilde{N}^{c}_2(t)$ respectively denote the number of times case (b), (c-1) and (c-2) occur  between time $C(\delta)+1$ and $t$. 

From \cref{lemma3_case_a}, we know that  with probability  at least $1-\delta$, $\forall t> C(\delta)$, MultiRUCB will not carry out Line 8  in Algorithm \ref{algorithm_MultiRUCB} but one of case (a), case (b) and case (c). 
In addition, with probability  at least $1-\delta$, $\forall t> C(\delta)$, if MultiRUCB carries out case (a), the comparison set will be $\mathcal{A}_t=\{x_1\}$. Thus, in order to bound the expected regret  with probability  at least $1-\delta$, it suffices to bound $\widetilde{N}^{b}(t)$, $\widetilde{N}^{c}_1(t)$ and  $\widetilde{N}^{c}_2(t)$.

\begin{lemma}
	\label{lemma4_case_b_case_c-1}
	With probability  at least $1-\delta$, $\forall t> C(\delta)$, 
	\begin{align}
	\widetilde{N}^{b}(t)+ \widetilde{N}^{c}_1(t) \leq  \sum \limits_{i>1} \frac{4\alpha}{\Delta^2_i} \ln t. \nonumber
	\end{align}
\end{lemma}

\begin{proof} According to \cref{lemma3_case_a}, with probability  at least $1-\delta$, $\forall t> C(\delta)$,  every time  case (b) occurs, we can observe at least one  outcome of duel between $x_1$ and some $x_i \  (i>1)$ ($\sum \limits_{i>1} \widetilde{N}_{1i}(t)$ increments by 1). Every time  case (c-1) occurs, we can observe outcomes of  $m-1$  duels between  $x_1$ and $x_i \  (i>1)$ ($\sum \limits_{i>1} \widetilde{N}_{1i}(t)$ increments by $m-1$).

	In the following we prove that with probability  at least $1-\delta$, $\forall t> C(\delta), i>1$, $\widetilde{N}_{1i}(t) \leq    \frac{4\alpha}{\Delta^2_i} \ln t $.
	
	Assume that $\exists t> C(\delta)$, $\exists i>1$. $\widetilde{N}_{1i}(t) > \frac{4\alpha}{\Delta^2_i} \ln t$. Let  $s$ denote the last time when we observed the dueling  outcome between $x_1$ and  $x_i$ up to time $t$, which implies $\widetilde{N}_{1i}(s)=\widetilde{N}_{1i}(t),  C(\delta) < s \leq t$. Using \cref{lemma2_CH_bound}, we  have that with probability  at least $1-\delta$,
	\begin{align}
	2 \sqrt{\frac{\alpha \ln s}{N_{1i}(s)}} \leq 2 \sqrt{\frac{\alpha \ln s}{\widetilde{N}_{1i}(s)}} \leq 2 \sqrt{\frac{\alpha \ln t}{\widetilde{N}_{1i}(t)}} < \Delta_i \nonumber
	\\
	u_{i1}(s) \leq p_{i1}+2 \sqrt{\frac{\alpha \ln s}{N_{1i}(t)}} < p_{i1}+\Delta_i=\frac{1}{2}.
	\nonumber
	\end{align} 
	Since $u_{i1}(s) < \frac{1}{2}$, $x_i$ cannot be in $\mathcal{C}$ and thus cannot be chosen into $\mathcal{A}_t$, which yields a contradiction.
	
	Taking summation over $i>1$, we have that with probability  at least $1-\delta$, $\forall t> C(\delta)$, $ \sum \limits_{i>1} \widetilde{N}_{1i}(t) \leq   \sum \limits_{i>1} \frac{4\alpha}{\Delta^2_i} \ln t$.

	Because with probability  at least $1-\delta$, $\forall t> C(\delta)$, each occurrence of  case (c-1) increments $\sum \limits_{i>1} \widetilde{N}_{1i}(t)$  by $m-1$, we can  bound the $\widetilde{N}^{c}_1(t)$ by 
	\begin{align}
	(m-1)\widetilde{N}^{c}_1(t) \leq & \sum \limits_{i>1} \widetilde{N}_{1i}(t) \leq   \sum \limits_{i>1} \frac{4\alpha}{\Delta^2_i} \ln t \nonumber 
	\\
	\widetilde{N}^{c}_1(t) \leq & \sum \limits_{i>1} \frac{4\alpha}{(m-1)\Delta^2_i} \ln t. \label{eq_c-1}
	\end{align}
	\cref{eq_c-1} will be used later (proof of \cref{lemma7_later_term}). 
	
	For $\widetilde{N}^{b}(t) + \widetilde{N}^{c}_1(t)$,  we have that with probability at least $1-\delta$, $\forall t> C(\delta)$, $\widetilde{N}^{b}(t) + \widetilde{N}^{c}_1(t) \leq \sum \limits_{i>1} \widetilde{N}_{1i}(t) \leq   \sum \limits_{i>1} \frac{4\alpha}{\Delta^2_i} \ln t$.
	
\end{proof}

\begin{lemma}
	\label{lemma5_case_c-2}
	With probability  at least $1-\delta$, $\forall t> C(\delta)$,  
	\begin{align}
	\widetilde{N}^{c}_2(t) \leq \sum \limits_{1<i<j} \frac{4\alpha}{C_m^2 \Delta^2_{ij}} \ln t.  \nonumber
	\end{align}
\end{lemma}

\begin{proof}
	$\forall t> C(\delta)$, every time  case (c-2) occurs, we will observe outcomes of  $C_m^2:=\frac{m(m-1)}{2}$  different duels between suboptimal arms, \emph{i.e.}, $\sum \limits_{1<i<j} \widetilde{N}_{ij}(t)$ will increment by $C_m^2$.

	In the following  we prove that with probability  at least $1-\delta$, $\forall t> C(\delta)$, $i,j>1$, $i \neq j $, $\widetilde{N}_{ij}(t) \leq    \frac{4\alpha}{\Delta^2_{ij}} \ln t $.
	
	Assume that $\exists t> C(\delta)$, $\exists i,j>1, \   i \neq j  \   (\text{wlog } i<j)$, $\widetilde{N}_{ij}(t) > \frac{4\alpha}{\Delta^2_{ij}} \ln t$. Let  $s$ denote the last time when we observed the dueling  outcome  between $x_i$ and $x_j$ up to time $t$, which implies $\widetilde{N}_{ij}(s)=\widetilde{N}_{ij}(t),  C(\delta) < s \leq t$. Using \cref{lemma2_CH_bound}, we have that with probability  at least $1-\delta$,
	\begin{align}
	2 \sqrt{\frac{\alpha \ln s}{N_{ij}(s)}} \leq 2 \sqrt{\frac{\alpha \ln s}{\widetilde{N}_{ij}(s)}} \leq 2 \sqrt{\frac{\alpha \ln t}{\widetilde{N}_{ij}(t)}} < \Delta_{ij} \nonumber
	\\
	u_{ji}(s) \leq p_{ji}+2 \sqrt{\frac{\alpha \ln s}{N_{ij}(t)}} < p_{ji}+\Delta_{ij}=\frac{1}{2}.  \nonumber
	\end{align} 
	Since $u_{ji}(s) < \frac{1}{2}$, $x_j$ cannot be in $\mathcal{C}$ and thus cannot be chosen into $\mathcal{A}_t$, which yields a contradiction.
	
	Taking summation over $1<i<j$, we have that with probability  at least $1-\delta$, $\forall t> C(\delta)$, $\sum \limits_{1<i<j} \widetilde{N}_{ij}(t) \leq   \sum \limits_{1<i<j} \frac{4\alpha}{\Delta^2_{ij}} \ln t$.
	Thus, with probability at least $1-\delta$, $\forall t> C(\delta)$, $ C_m^2 \widetilde{N}^{c}_2(t) \leq \sum \limits_{1<i<j} \widetilde{N}_{ij}(t) \leq   \sum \limits_{1<i<j} \frac{4\alpha}{\Delta^2_{ij}} \ln t$.
	This concludes the proof of \cref{lemma5_case_c-2}.
	
\end{proof}

\begin{lemma}
	\label{lemma6_former_term}
	With probability  at least $1-\delta$, for any time $T$, 
	\begin{align}
	N(T) \leq C(\delta)+ \sum \limits_{i>1} \frac{4\alpha}{\Delta^2_i} \ln T + \sum \limits_{1<i<j} \frac{4\alpha}{C_m^2 \Delta^2_{ij}} \ln T. \nonumber
	\end{align}
\end{lemma}

\begin{proof}
	\cref{lemma6_former_term} holds by combining \cref{lemma4_case_b_case_c-1} and \cref{lemma5_case_c-2}.
\end{proof}

\cref{lemma6_former_term} gives a high probability bound of $N(T)$. In the following we will give another high probability bound (\cref{lemma7_later_term}) of  $N(T)$ using the choice strategy of case (c).

Before stating \cref{lemma7_later_term}, we firstly introduce a definition, which will be used in the proof of \cref{lemma7_later_term}.

\begin{definition}
	\label{definition1}
	Let $\widehat{T}_{\delta}$ be the smallest time satisfying
	\begin{align}
	\widehat{T}_{\delta}> C(\frac{\delta}{2}) +  \sum \limits_{i>1} \frac{4\alpha}{\Delta^2_i} \ln \widehat{T}_{\delta} + \sum \limits_{1<i<j} \frac{4\alpha}{C_m^2 \Delta^2_{ij}} \ln \widehat{T}_{\delta}. \nonumber
	\end{align}
	where $\widehat{T}_{\delta}$ is guaranteed to exist because the left side of the inequality grows linearly with $\widehat{T}_{\delta}$ and the right side grows logarithmically.
\end{definition}

In the following we prove a upper bound of $\widehat{T}_{\delta}$ using similar techniques in \cite{Relative_upper_confidence_bound}.

Define $C:=C(\frac{\delta}{2}),\  D:=\sum \limits_{i>1} \frac{4\alpha}{\Delta^2_i}+ \sum \limits_{1<i<j} \frac{4\alpha}{C_m^2 \Delta^2_{ij}} $. To find a upper bound of $\widehat{T}_{\delta}$, we need to  produce one number $T$ satisfying  $T>C+D \ln T$. It is easy to prove one such number is $T=2C+2D \ln 2D$.
\begin{align}
C+D \ln (2C+2D \ln 2D)  \overset{a}{\leq} & C+ D \ln (2D \ln 2D) \nonumber
\\
& + D \frac{2C}{2D \ln 2D} \nonumber
\\
\overset{b}{\leq} & C+ D \ln ((2D)^2)   \nonumber
\\
& +  \frac{C}{ \ln 2D} \nonumber
\\
\overset{c}{\leq} & 2C+ 2D \ln 2D, \nonumber
\end{align}
where (a) uses a first order Taylor expansion, (b) uses $\ln 2D <2D$ and (c) uses $D>2$.

Thus, we can bound  $\widehat{T}_{\delta}$ by 
\begin{align}
\widehat{T}_{\delta} \leq 2C+2D \ln 2D. \label{eq_lemma7_T_delta}
\end{align}

\begin{lemma}
	\label{lemma7_later_term}
	With probability at least $1-\delta$,   for any time $T$,  
	\begin{align}
	N(T) \leq & 2C+2D \ln 2D +4 \ln \frac{2}{\delta} +   \sum \limits_{i>1} \frac{4\alpha}{\Delta^2_i} \ln T \nonumber 
	\\
	& +   2\sum \limits_{i>1} \frac{4\alpha}{(m-1)\Delta^2_i} \ln T. \nonumber
	\end{align}
	where $C:=C(\frac{\delta}{2}),\  D:=\sum \limits_{i>1} \frac{4\alpha}{\Delta^2_i}+ \sum \limits_{1<i<j} \frac{4\alpha}{C_m^2 \Delta^2_{ij}} $.
\end{lemma}

\begin{proof} From \cref{lemma3_case_a,lemma4_case_b_case_c-1,lemma5_case_c-2} and \cref{definition1}, we have that with probability  at least $1-\frac{\delta}{2}$, there exists a time $T_{\delta} \in (C(\frac{\delta}{2}), \widehat{T}_{\delta} ]$ when case (a) occurs. This implies that with probability  at least $1-\frac{\delta}{2}$, $\mathcal{B}$ has been set as $\mathcal{B} =\{x_1\}$ from time $T_{\delta}$ on. Thus, from  time $T_{\delta}$ on, if MultiRUCB carries out case (c), case(c-1) will occur with probability of $\frac{1}{2}$.
	
	Let  $\widehat{N}^{b}(t)$, $\widehat{N}^{c}_1(t)$ and  $\widehat{N}^{c}_2(t)$ denote the number of times case (b), (c-1) and (c-2) occur between time $T_{\delta}+1$ and $t$, respectively.
	We also introduce the following two sets of random variables:
	\begin{itemize}
		\item $\tau_0,\tau_1,\tau_2,...,$ where $\tau_0:=T_{\delta}$ and $\tau_l$ is the $l^{th}$ time case (c-1) occurs after time $T_{\delta}$.
		\item $n_1,n_2,...,$ where $n_l$ is the number of times case (c-2) occurs between $\tau_{l-1}$ and $\tau_l$.
	\end{itemize}
	
	Using \cref{lemma4_case_b_case_c-1},  we have that with probability at least $1-\frac{\delta}{2}$, $\forall t> T_{\delta}$, 
	\begin{align}
	\widehat{N}^{b}(t)+ \widehat{N}^{c}_1(t) \leq \widetilde{N}^{b}(t)+ \widetilde{N}^{c}_1(t) \leq  \sum \limits_{i>1} \frac{4\alpha}{\Delta^2_i} \ln t.  \label{eq_lemma7_case_b_and_case_c-1}
	\end{align}

	Using \cref{eq_c-1}, we have that with probability at least $1-\frac{\delta}{2}$, $\forall t> T_{\delta}$, 
	\begin{align}
	\widehat{N}^{c}_1(t)\leq \widetilde{N}^{c}_1(t) \leq  \sum \limits_{i>1} \frac{4\alpha}{(m-1)\Delta^2_i} \ln t. \nonumber
	\end{align}

	This means that with probability at least $1-\frac{\delta}{2}$, between time $T_{\delta}+1$ and $t$, case (c-1) occurs at most $L^{c}_1(t):= \sum \limits_{i>1} \frac{4\alpha}{(m-1)\Delta^2_i} \ln t $ times. Moreover,
	with probability at least $1-\frac{\delta}{2}$,  for any time  $t>T_{\delta}$, if case (c-1) has occurred  $L^{c}_1(t)$ times, all suboptimal arms $x_i \  (i>1)$ satisfy $u_{i1}<\frac{1}{2}$ and  case (c-2) cannot occur.
	Thus, we have that with probability at least $1-\frac{\delta}{2}$, $\forall t> T_{\delta}$,
	\begin{align}
	\widehat{N}^{c}_2(t) \leq \sum \limits_{l=1} \limits^{L^{c}_1(t)} n_l. \nonumber
	\end{align}
	
	To bound the sum of intervals $n_l$, we introduce i.i.d. geometric random variables $\{\hat{n}_l\}_{l=1,2,...,r} $ with parameter $\frac{1}{2}$.  $\hat{n}_l$ bounds $n_l$ because $n_l$ counts the number of times it takes for case (c) to produce one case (c-1). 
	
	Using similar techniques in  \cite{feller1957_geometric_distribution,Relative_upper_confidence_bound} to bound the sum of $\{\hat{n}_l\}_{l=1,2,...,r} $, which we denote by $n$, we can obtain that  with probability at least $1-\frac{\delta}{2}$, 
	\begin{align}
	n < 2r+4 \ln \frac{2}{\delta}, \nonumber
	\end{align}
	Note that this  $1-\frac{\delta}{2}$ is different from the aforementioned  $1-\frac{\delta}{2}$ that is derived from \cref{lemma2_CH_bound}.  
	
	Setting $r=L^{c}_1(t)$, we have that with probability at least $1-\delta$, $\forall t> T_{\delta}$,
	\begin{align}
	\widehat{N}^{c}_2(t) & \leq \sum \limits_{l=1} \limits^{L^{c}_1(t)} n_l \nonumber
	\\
	& \leq 2\sum \limits_{i>1} \frac{4\alpha}{(m-1)\Delta^2_i} \ln t +4 \ln  \frac{2}{\delta}. \label{eq_lemma7_case_c-2}
	\end{align}
	
	Taking summation over $T_{\delta}$ (\cref{eq_lemma7_T_delta}), $\widehat{N}^{b}(t)$, $\widehat{N}^{c}_1(t)$ (\cref{eq_lemma7_case_b_and_case_c-1}) and  $\widehat{N}^{c}_2(t)$ (\cref{eq_lemma7_case_c-2}), we obtain the result of \cref{lemma7_later_term}.
	
\end{proof}


\thmMultiRUCB*

\begin{proof} Combining \cref{lemma6_former_term} and \cref{lemma7_later_term}, we can obtain that with probability  at least $1-\delta$, for any time $T$, 
	\begin{align}
	N(T) \leq  \min \Bigg\{  C(\delta)+ D \ln T, \   2C(\frac{\delta}{2})+2D \ln 2D \nonumber
	\\
	+ 4 \ln \frac{2}{\delta}  	 
	+ \sum \limits_{i>1} \frac{4\alpha}{\Delta^2_i} \ln T
	+ 2\sum \limits_{i>1} \frac{4\alpha}{(m-1)\Delta^2_i} \ln T \Bigg \}, \nonumber
	\end{align}

	Integrating $N(T)$ with respect to $\delta$ from $0$ to $1$, we have that given $\alpha>1$,  $\mathbb{E}[N(T)]$ is bounded by
	\begin{align}
	\mathbb{E}  [N(T)]  \leq &  \left[ \left(\frac{2(4\alpha -1)K^2}{2\alpha -1} \right)^{\frac{1}{2\alpha -1}}  \frac{2\alpha -1}{\alpha -1} \right] \nonumber \\
	& +  \min  \Bigg\{   D \ln T, \nonumber
	\\
	& 2D \ln 2D + 8 +  \frac{m+1}{m-1} \sum \limits_{i>1} \frac{4\alpha}{\Delta^2_i} \ln T \Bigg\}. \nonumber
	\end{align}
	
	\cref{theorem3_MultiRUCB} is obtained by applying
	\begin{align}
	\mathbb{E}[R_T] \leq \mathbb{E}[N(T)] \cdot \Delta_{max}. \nonumber
	\end{align}
	
\end{proof}

\section*{D \quad  Variance Results in Multi-dueling Bandits Experiments}
In this section, we present the omitted variance results in the multi-dueling bandits experiments (See \cref{section_multiexp}). 
\cref{table_variance} shows the variances of cumulative regrets at the $10^6$ timestep for 50 independent runs. Columns 2-5 correspond to the experiments in Figure \ref{fig_multi-dueling} (a-d). ``Syn'' refers to the synthetic dataset. 

\begin{table}[h]
	\centering
	\caption{The variance results in the multi-dueling bandit experiments.}
	\label{table_variance}
	\renewcommand\arraystretch{1.0}
	\normalsize
	\setlength{\tabcolsep}{2.2pt}
	\begin{tabular}{ccccc}
		\toprule
		Algorithms&Syn, m=8&Syn, m=16&MSLR, m=8&MSLR, m=16\\
		\midrule	
		\textbf{MultiRUCB}&1648.05&1042.07&678.42&231.97\\
		IndSelfSparring&2047.69&1243.43&741.39&264.67\\
		MDB&2142.22&1246.30&754.21&277.62\\
		MultiSparring&2201.16&1328.71&864.56&392.61\\
		\bottomrule
	\end{tabular}
\end{table}

\end{document}